\documentclass{article}

\usepackage{microtype}
\usepackage{graphicx}
\usepackage{subcaption}
\usepackage{booktabs} % for professional tables

\usepackage{hyperref}

\usepackage[ruled,vlined]{algorithm2e}
\usepackage{algpseudocode}
\usepackage[normalem]{ulem}
\usepackage{multirow}
\usepackage{caption}
\usepackage{enumitem}
\usepackage{float}
\usepackage[accepted]{icml2019}

\usepackage{amsmath,amsfonts,bm}

\def\1{\bm{1}}

\DeclareMathAlphabet{\mathsfit}{\encodingdefault}{\sfdefault}{m}{sl}
\SetMathAlphabet{\mathsfit}{bold}{\encodingdefault}{\sfdefault}{bx}{n}

\usepackage{amsthm}
\newtheorem{theorem}{Theorem}%[section]
\newtheorem{lemma}[theorem]{Lemma}

\newtheorem{definition}[theorem]{Definition}

\newtheorem{proposition}[theorem]{Proposition}

\newcommand{\xref}[1]{\S\ref{#1}}

\usepackage{mathtools}

\newcommand{\header}[1]{\vspace{0.05in}\noindent\textbf{#1}}

\newtheorem*{proposition2*}{Proposition 2}
\newtheorem*{lemma3*}{Lemma 2}
\newtheorem*{lemma4*}{Lemma 3}
\newtheorem*{theorem5*}{Theorem 2}
\newtheorem*{theorem6*}{Theorem 3}
\newtheorem*{lemma7*}{Lemma 4}
\icmltitlerunning{Functional Transparency for Structured Data: a Game-Theoretic Approach}

\begin{document}

\twocolumn[
\icmltitle{Functional Transparency for Structured Data: a Game-Theoretic Approach}

% It is OKAY to include author information, even for blind
% submissions: the style file will automatically remove it for you
% unless you've provided the [accepted] option to the icml2019
% package.

% List of affiliations: The first argument should be a (short)
% identifier you will use later to specify author affiliations
% Academic affiliations should list Department, University, City, Region, Country
% Industry affiliations should list Company, City, Region, Country

% You can specify symbols, otherwise they are numbered in order.
% Ideally, you should not use this facility. Affiliations will be numbered
% in order of appearance and this is the preferred way.
% \icmlsetsymbol{equal}{*}

\begin{icmlauthorlist}
\icmlauthor{Guang-He Lee}{to}
\icmlauthor{Wengong Jin}{to}
\icmlauthor{David Alvarez-Melis}{to}
\icmlauthor{Tommi S. Jaakkola}{to}
\end{icmlauthorlist}

\icmlaffiliation{to}{MIT Computer Science \& Artificial Intelligence Lab}

\icmlcorrespondingauthor{Guang-He Lee}{guanghe@csail.mit.edu}

% You may provide any keywords that you
% find helpful for describing your paper; these are used to populate
% the "keywords" metadata in the PDF but will not be shown in the document
\icmlkeywords{Machine Learning, ICML}

\vskip 0.3in
]

% this must go after the closing bracket ] following \twocolumn[ ...

% This command actually creates the footnote in the first column
% listing the affiliations and the copyright notice.
% The command takes one argument, which is text to display at the start of the footnote.
% The \icmlEqualContribution command is standard text for equal contribution.
% Remove it (just {}) if you do not need this facility.

\printAffiliationsAndNotice{}  % leave blank if no need to mention equal contribution
% \printAffiliationsAndNotice{\icmlEqualContribution} % otherwise use the standard text.

\begin{abstract}
We provide a new approach to training neural models to exhibit transparency in a well-defined, functional manner. Our approach naturally operates over structured data and tailors the predictor, functionally, towards a chosen family of (local) witnesses. The estimation problem is setup as a co-operative game between an unrestricted \emph{predictor} such as a neural network, and a set of \emph{witnesses} chosen from the desired transparent family. The goal of the witnesses is to highlight, locally, how well the predictor conforms to the chosen family of functions, while the predictor is trained to minimize the highlighted discrepancy. We emphasize that the predictor remains globally powerful as it is only encouraged to agree locally with locally adapted witnesses. We analyze the effect of the proposed approach, provide example formulations in the context of deep graph and sequence models, and empirically illustrate the idea in chemical property prediction, temporal modeling, and molecule representation learning.
\end{abstract}

\section{Introduction}
\label{sec:intro}

\begin{figure*}[t]
\centering
\begin{subfigure}{0.48\textwidth}
  \centering
  \includegraphics[width=0.8\textwidth]{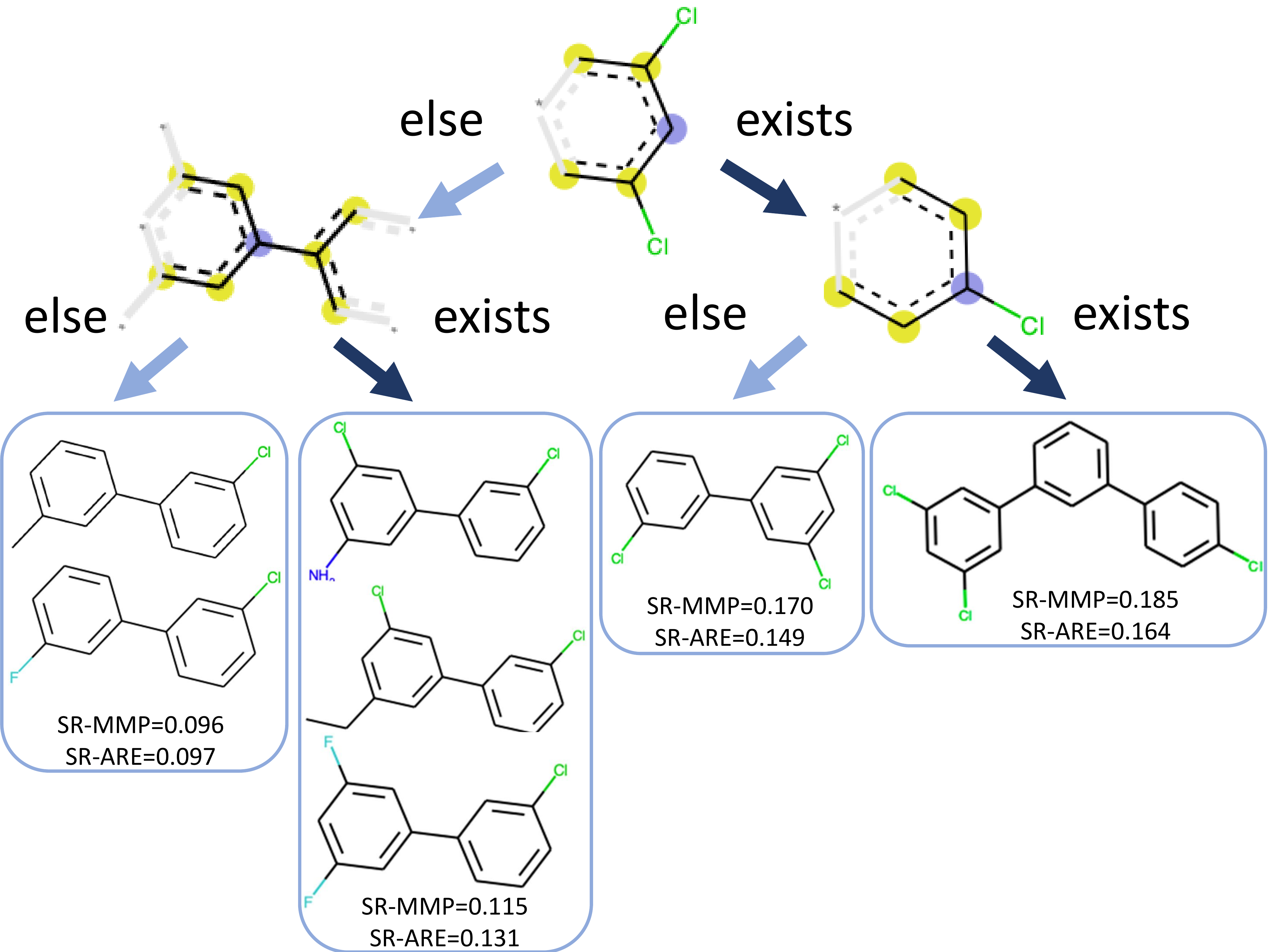}
  \caption{The explanation from our model (trained for transparency).}\label{fig:loclin}
\end{subfigure}%
~
\begin{subfigure}{0.48\textwidth}
  \centering
  \includegraphics[width=0.8\linewidth]{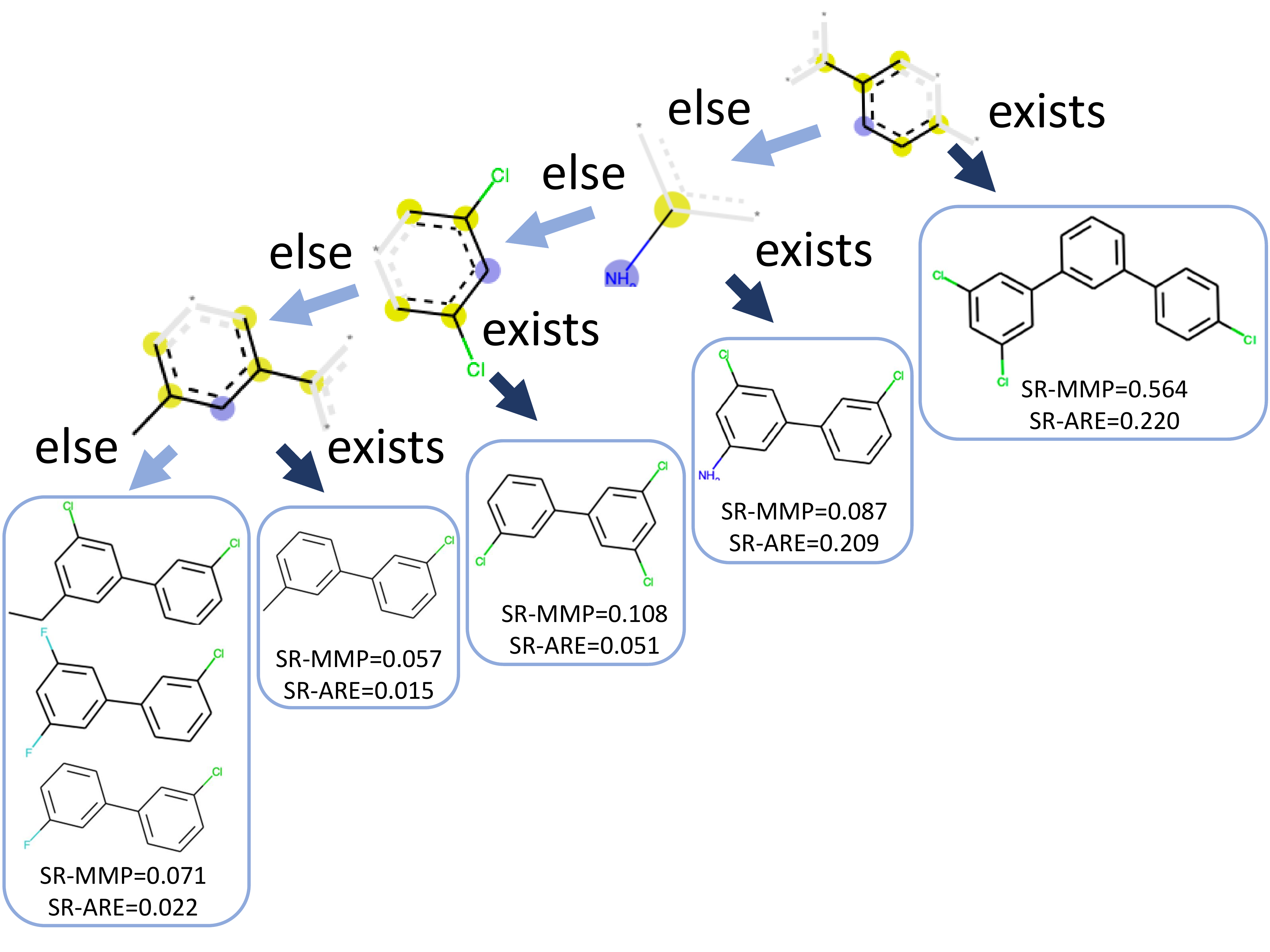}
  \caption{The explanation from a normally trained model.}\label{fig:loccir}
\end{subfigure}
\vspace{-1mm}
\caption{During testing, we fit decision trees to our model and an unregularized model on molecule property prediction at the same local neighborhood such that the functional approximations are comparable in AUC (because the scale is not crucial). The split criterion on each node is based on the existence of a \emph{complete chemical substructure} in Morgan fingerprints~\cite{rogers2010extended}. The color of each Morgan fingerprint simply reflects the radius of the fingerprint.}
\label{fig:teaser}
\vspace{-3mm}
\end{figure*}

Modern machine learning tasks are increasingly complex, requiring flexible models with large numbers of parameters such as deep networks~\citep{silver2016mastering, vaswani2017attention, huang2017densely}. Such modeling gains often come at the cost of transparency or interpretability. This is particularly problematic when predictions are fed into decision-critical applications such as medicine where the ability to verify predictions may be just as important as the raw predictive power.

It seems plausible to guide a flexible neural network towards a complex yet well-understood (i.e., transparent) functional class.
For example, in realizing Wasserstein-1 distance~\citep{arjovsky2017wasserstein}, the discriminator should be limited to 1-Lipschitz functions. 
A strict adherence to a complex, global functional class is not the only way to achieve transparency. 
For example, linearity is a desirable characteristic for transparency but is sensible to enforce only locally. 
We offer therefore a new notion of transparency -- functional transparency -- where the goal is to guide models to adopt a desirable local behavior yet allowing them to be more flexible globally. 
Note that functional transparency should be established only approximately in many cases since, e.g., strict local linearity implies global linearity. %See Figure~\ref{fig:teaser} for an illustration. 

Previous approaches to interpretability have mainly focused on models that operate on fixed-size data, such as scalar-features~\cite{lakkaraju2016interpretable} or image prediction~\cite{selvaraju2016grad, mahendran2015understanding}. The emphasis has been on feature relevance or selection~\citep{Ribeiro2016Why}. Recent methods do address some of the challenges in sequential data \cite{Lei2016Rationalizing,arras2017relevant}, primarily in NLP tasks where the input sequence is discrete. Interpretability for continuous temporal data~\cite{al2017contextual, wu2018tree} or graph structures remains largely unexplored.

We develop a novel approach to transparency that is naturally suited for structured data. At the core of our approach is a game-theoretic definition of transparency. This is set up as a two-player co-operative game between a \emph{predictor} and a \emph{witness}. The predictor remains a complex model whereas the witness is chosen from a simple transparent family. Transparency arises from the fact that the predictor is encouraged to exemplify simple behavior as captured by the witness in each local region while remaining globally powerful. The approach differs from global regularization of models towards interpretability~\cite{wu2018tree}, models that are constructed a priori to be interpretable, either architecturally or in terms of the function class~\cite{al2017contextual,Lei2016Rationalizing}, or from post-hoc explanations of black-box methods via local perturbations~\cite{Ribeiro2016Why, alvarez2017causal}. Our models are guided towards functional transparency during learning. 

As an illustration, we contrast our approach with methods that seek to obtain interpretable explanations after the fact (e.g., \citep{Ribeiro2016Why}). Derived explanation after training can be misleading in some cases if the explanation does not match the functional behavior of the model. For example, Figure~\ref{fig:teaser} shows local decision tree approximations for two models: our model trained with such local witnesses (a, left), and an unregularized model (b, right). The trees are constructed to achieve the same level of approximation. The tree for the unregularized model only filters one sample in each split, lacking generality to explain the (local) behavior. This phenomenon is related to unstable explanations that arise with already trained models~\cite{alvarez2018robustness, ghorbani2017interpretation}. %, which has recently been identified as an open problem in interpretability. 

The game theoretic approach is very flexible in terms of models and scenarios. We therefore illustrate the approach across a few novel scenarios: explaining graph convolutional models using decision trees, revealing local functional variation of a deep sequence model, and exemplifying decision rules for the encoder in unsupervised graph representation learning. Our main contributions are: 
\begin{itemize}
	\vspace{-2mm}
	\itemsep0em
	\item A novel game-theoretic approach to transparency, applicable to a wide range of prediction models, architectures, and local transparency classes, without requiring differentiability. 
    %\item Accurate yet transparent predictors which are trained coordinately with the local witnesses so as to balance functional transparency and accuracy.
    \item Analysis on the effective size of the local regions and establishing equilibria pertaining to different game formulations.
    \item Illustration of deep models across several tasks, from chemical property prediction, physical component modeling, to molecule representation learning. 
    %with decision tree and linear explanations.
    % \item Empirical study on the trade-off between prediction accuracy and functional transparency.
\end{itemize}
\vspace{-1mm}
%%%%%%%%%%%%%%%%%%%%%%%%%%%%%%%%%%%%%%%%%%%%%%%%%%%%%%%%%%%%
\vspace{-2mm}
\section{Related Work}
\vspace{-1mm}
\label{sec:related}

The role of transparency is to expose the inner-workings of an algorithm~\citep{citron2014scored, pasquale2015black}, such as decision making systems. This is timely for state-of-the-art machine learning models that are typically over-parameterized~\citep{silver2016mastering, he2016deep} and therefore effectively black-box models. An uncontrolled model is also liable to various attacks~\citep{goodfellow14}. 

Our goal is to regularize a complex deep model so that it exhibits a desired local behavior. The approach confers an approximate operational guarantee rather than directly interpretability. In contrast, examples of archetypal interpretable models include linear classifiers, decision trees~\citep{quinlan2014c4}, and decision sets~\citep{lakkaraju2016interpretable}; recent approaches also guide complex models towards highlighting pieces of input used for prediction~\citep{Lei2016Rationalizing}, grounding explanations via graphical models~\citep{al2017contextual}, or generalizing linear models while maintaining interpretability~\citep{alvarez2018towards}. A model conforming to a known functional behavior, at least locally, as in our approach, is not necessarily itself human-interpretable. The approximate guarantee we offer is that the complex model indeed follows such a behavior and we also quantify to what extent this guarantee is achieved.

Previous work on approximating a functional class via neural networks can be roughly divided into two types: parametrization-based and regularization-based methods. Works in the first category seek self-evident adherence to a functional class, which include maintaining Lipschitz continuity via weight clipping~\citep{arjovsky2017wasserstein}, orthogonal transformation via scaled Cayley transform of skew-symmetric matrices~\citep{helfrich2017orthogonal}, and ``stable'' recurrent networks via spectral norm projection on the transition matrix~\citep{miller2018recurrent}.

A softer approach is to introduce a regularization problem that encourages neural networks to match properties of the functional class. 
Such regularization problem might come in the form of a gradient penalty as used in several variants of GAN~\citep{gulrajani2017improved, bellemare2017cramer, mroueh2017sobolev} under the framework of integral probability metrics (IPM)~\citep{muller1997integral}, layer-wise regularization of transformation matrices~\citep{cisse2017parseval} towards parseval tightness~\citep{kovavcevic2008introduction} for robustness, and recent adversarial approaches to learn representations for certain independence statements~\citep{ganin2016domain, zhao2017learning}.
Typically, a tailored regularization problem is introduced for each functional class. Our work follows this general theme in the sense of casting the overall problem as a regularization problem. However, we focus on transparency and our approach -- a general co-operative game -- is quite different. Our methodology is applicable to any choice of (local) functional class without any architectural restrictions on the deep model whose behavior is sculpted. The optimization of functional deviation in the game must remain tractable, of course.

\vspace{-1mm}
\section{Methodology}
\vspace{-1mm}
\label{sec:method}

In this work, given a dataset $\mathcal{D} = \{(x_i, y_i)\}_{i=1}^N\subset \mathcal{X} \times \mathcal{Y}$, we learn an (unrestricted) predictive function $f \in \mathcal{F}: \mathcal{X} \to \mathcal{Y}$ together with a transparent -- and usually simpler -- function $g \in \mathcal{G}:\mathcal{X}\to\mathcal{Y}$ defined over a functional class $\mathcal{G}$. We refer to functions $f$ and $g$ as the \textit{predictor} and the \textit{witness}, respectively, throughout the paper. 
Note that we need not make any assumptions on the functional class $\mathcal{F}$, instead allowing a flexible class of predictors. 
In contrast, the family of witnesses $\mathcal{G}$ is strictly constrained to be a \textit{transparent} functional set, such as the set of linear functions or decision trees. 
% \sout{to use}
We assume to have a deviation function $d: \mathcal{Y} \times \mathcal{Y} \to \mathbb{R}_{\geq 0}$ such that $d(y, y') = 0 \iff y = y'$, which measures discrepancy between two elements in $\mathcal{Y}$ and can be used to optimize $f$ and $g$. 
To simplify the notation, we define $\mathcal{D}_x := \{x_i: (x_i, y_i) \in \mathcal{D}\}$. 
We introduce our game-theoretic framework in \xref{sec:gt_transp}, analyze it in \xref{seq:analysis}, and instantiate the framework with concrete models in \xref{sec:prob}.

\vspace{-1mm}
\subsection{Game-Theoretic Transparency}\label{sec:gt_transp}
\vspace{-1mm}

There are many ways to use a witness function $g \in \mathcal{G}$ to guide the predictor $f$ by means of discrepancy measures. 
However, since the witness functions can be weak such as linear functions, we cannot expect that a reasonable predictor would agree to it globally. Instead, we make a slight generalization to enforce this criterion only locally, over different sets of neighborhoods. 
To this end, we define \emph{local} transparency by measuring how close $f$ is to the family $\mathcal{G}$ over a local neighborhood $\mathcal{B}(x_i) \subset \mathcal{X}$ around an observed point $x_i$. 
% One straightforward instantiation of such a neighborhood could be the intersection between an $\ell_p$-norm ball around $x_i$ and $\mathcal{D}_x$.
One straightforward instantiation of such a neighborhood $\mathcal{B}_\epsilon({x_i})$ in temporal domain will be simply a local window of points $\{x_{i-\epsilon},\dots, x_{i+\epsilon}\}$.
%$\{x_t\}_{t=\max(1, i-\epsilon)}^{\min(T, i+\epsilon)}$. %\footnote{An $\ell_p$-norm ball around $x_i$ would also fit our derivation, but disregards the temporal dependency.}. 
Our resulting local discrepancy measure is
\begin{align}
\vspace{1mm}
%d_{\mathcal{B}(x_i)}(f, \mathcal{G}) := 
\min_{g \in \mathcal{G}} \frac{1}{|\mathcal{B}(x_i)|} \sum_{x_j \in \mathcal{B}(x_i)} \!\!\! d(f(x_j), g(x_j)). %\nonumber
\label{eq:local_deviation}
\vspace{1mm}
\end{align}
The summation can be replaced by an integral when a continuous neighborhood is used. 
The minimizing witness function, $\hat g_{x_i}$, is indexed by the point $x_i$ around which it is estimated; depending on the function $f$, the minimizing witness can change from one neighborhood to another. If we view the minimization problem game-theoretically, $\hat g_{x_i}$ is the \emph{best response strategy} of the local witness around $x_i$. 

The local discrepancy measure can be incorporated into an overall estimation criterion in many ways so as to guide the predictor towards the desired functional form. This guidance can be offered as a \emph{uniform} constraint with a permissible $\delta$-margin, as an additive \emph{symmetric} penalty, or defined \emph{asymmetrically} as a game theoretic penalty where the information sets for the predictor and the witness are no longer identical. We consider each of these in turn.
%into an overall regularization problem for the predictor either \emph{symmetrically}, as a shared objective, or \emph{asymmetrically}, where the goals differ between the predictor and the witness game-theoretically. 

%The local discrepancy measure can be subsequently incorporated into an overall regularization problem for the predictor either \emph{symmetrically}, as a shared objective, or \emph{asymmetrically}, where the goals differ between the predictor and the witness game-theoretically. 

{\bf Uniform criterion.} 
A straightforward formulation is to confine $f$ to remain within a margin $\delta$ of the best fitting witness for every local neighborhood. Assume that a primal loss $\mathcal{L}(\cdot, \cdot)$ is given for a learning task. The criterion imposes the $\delta$-margin constraint uniformly as
\begin{align}
    & \sum_{(x_i, y_i) \in \mathcal{D}} \mathcal{L}(f(x_i), y_i) \label{eq:uniform:game}\\
    & s.t. \min_{g \in \mathcal{G}} \frac{1}{|\mathcal{B}(x_i)|} \sum_{x_j \in \mathcal{B} (x_i)} d(f(x_j), g(x_j)) \leq \delta, \forall x_i \in \mathcal{D}_x. \nonumber
\end{align}
We assume that the optimal $g$ with respect to each constraint may be efficiently found due to the simplicity of $\mathcal{G}$ and the regularity of $d(\cdot, \cdot)$. We also assume that the partial derivatives with respect to $f$, for fixed witnesses, can be computed straightforwardly under sufficiently regular $\mathcal{L}(\cdot, \cdot)$ in a Lagrangian form. In this case, we can solve for $f$, local witnesses, and the Lagrange multipliers using the mirror-prox algorithm~\cite{nemirovski2004prox}.

\iffalse % -------- TJ: this takes a lot of space but doesn't really add anything
To solve this problem, we first rewrite it in the Langrangian form as
\begin{align}
    \!\!\!\! \min_f & \max_{\xi \succeq 0} \sum_{(x_i, y_i) \in \mathcal{D}} \mathcal{L}(f(x_i), y_i) \nonumber\\
    \!\!\!\! & + \xi_i \bigg[ \min_{g \in \mathcal{G}} \frac{1}{|\mathcal{B} (x_i)|} \sum_{x_t \in \mathcal{B} (x_i)} \!\! d(f(x_t), g(x_t)) - \delta\bigg]. \label{eq:uniform:game}
\end{align}
\fi % -------

The hard constraints in the uniform criterion will lead to strict transparency guarantees. However, the effect may be undesirable in some cases where the observed data (hence the predictor) do not agree with the witness in all places. {The resulting loss of performance may be too severe.} As an alternative, we can enforce the agreement with local witnesses to be small in aggregate across neighborhoods. 

{\bf Symmetric game.} We define an additive, unconstrained, symmetric criterion to smoothly trade off between performance and transparency. %The goal of the predictor in this case is to find $f$ that offers the best balance between the primal loss and local transparency. 
% In a symmetric formulation, the goal of the predictor is then to find $f$ that offers the best balance between the primal loss and local transparency. 
The resulting objective is
\begin{align}
\sum_{(x_i, y_i) \in \mathcal{D}} \bigg[\mathcal{L}(f(x_i), y_i) +\hspace{1in}\mbox{}\nonumber \\
   \min_{g \in \mathcal{G}} \frac{\lambda}{|\mathcal{B}(x_i)|} \sum_{x_j \in \mathcal{B} (x_i)} d(f(x_j), g(x_j)) \bigg] \label{eq:coop_game}
\end{align}
\iffalse % ---
\begin{align}
\vspace{-1mm}
\!\!\sum_{i=1}^N \! \bigg[& \mathcal{L}(f(x_i), y_i) \!+\! \frac{\lambda}{|\mathcal{B}(x_i)|} \!\! \sum_{x_j \in \mathcal{B}(x_i)} \!\!\! d(f(x_j), g_{x_i}(x_j))\bigg] \!
\vspace{-1mm}
\label{eq:coop_game}
\end{align}
to be minimized with respect to both $f$ and $g_{x_i}$. Here, $\lambda$ is a hyper-parameter that must be set.
\fi % ---

To illustrate the above idea, we generate a synthetic dataset to show a \emph{neighborhood} in Figure~\ref{fig:data} with an unconstrained piecewise linear predictor $f\in \mathcal{F_{\text{piecewise linear}}}$ in Figure~\ref{fig:deep}. Clearly, $f$ does not agree with a linear witness within this neighborhood. However, when we solve for $f$ together with a linear witness $g_{x_i} \in \mathcal{G}_{\text{linear}}$ as in Figure~\ref{fig:linear}, the resulting function has a small residual deviation from $\mathcal{G}_\text{linear}$, more strongly adhering to the linear functional class while still closely tracking the observed data. Figure~\ref{fig:stump} shows the flexibility of our framework where a very different functional behavior can be induced by changing the functional class for the witness.

\begin{figure}	
	\begin{subfigure}{0.225\textwidth}
		\centering
		\includegraphics[width=1.\linewidth]{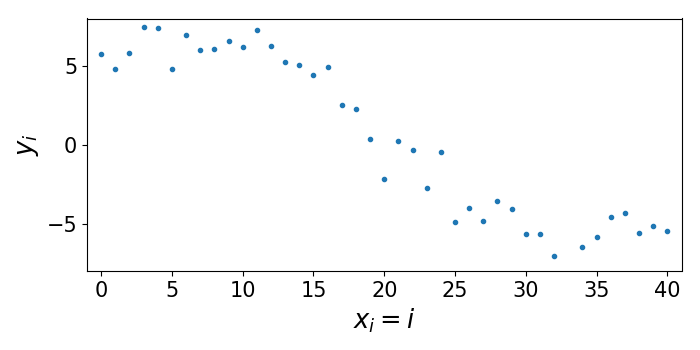}
% 		\vspace{-1mm}
		\caption{Neighborhood $\mathcal{B}(x_{20})$}\label{fig:data}
	\end{subfigure}
	~
	\begin{subfigure}{0.225\textwidth}
		\centering
		\includegraphics[width=1.\linewidth]{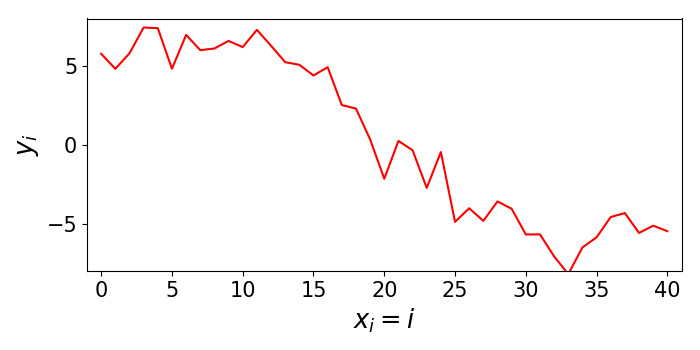}
% 		\vspace{-1mm}
		\caption{$f\in\mathcal{F}_\text{piecewise linear}$}\label{fig:deep}
	\end{subfigure}
    %~
    
	\begin{subfigure}{0.225\textwidth}
		\centering
		\includegraphics[width=1.\linewidth]{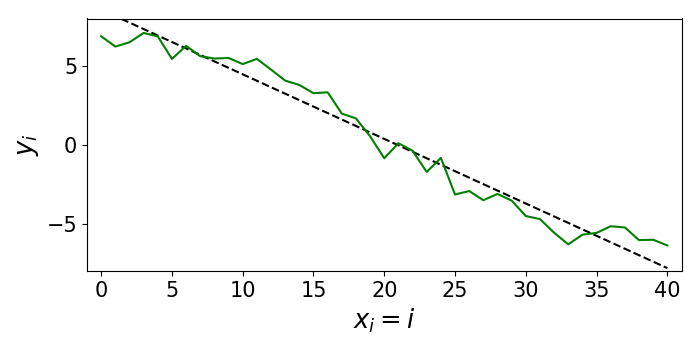}
% 		\vspace{-1mm}
		\caption{$g_{x_{i=1}} \in \mathcal{G}_{\text{linear}}$}\label{fig:linear}
	\end{subfigure}
	~
	\begin{subfigure}{0.225\textwidth}
		\centering
		\includegraphics[width=1.\linewidth]{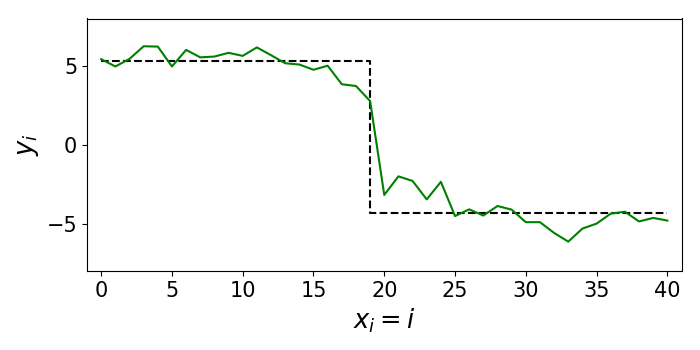}
% 		\vspace{-1mm}
		\caption{$g_{x_{i=1}} \in \mathcal{G}_{\text{decision stump}}$}\label{fig:stump}
	\end{subfigure}

	\caption{Examples of fitting a neighborhood $\mathcal{B}(x_{20})$~(\ref{fig:data}) with a piecewise linear predictor~(\ref{fig:deep}). Using different witness families (Figs.~\ref{fig:linear}\&\ref{fig:stump}, dashed lines) leads to predictors (solid green) with different behaviors, despite yielding the same error (MSE=1.026).
	%When playing with different families of witnesses~(Figs.~\ref{fig:linear}\&\ref{fig:stump}, dashed lines), the resulting predictors (in solid green) behave differently although they yield the same error (MSE=1.026).
    }\label{fig:synthetic}
    %The dashed line indicates the resulting strategy of each witness.}
    %\vspace{-1mm}
\end{figure}

{\bf Asymmetric game.}
Solving the symmetric criterion can be computationally inefficient since the predictor is guided by its deviation from each of the local witness on all points within each of the local neighborhoods. Moreover, the predictor value at any point $x_i$ is subject to potentially conflicting regularization terms across the neighborhoods, which is undesirable. The inner summation in Eq.~\eqref{eq:coop_game} may involve different sizes of neighborhoods $\mathcal{B}(x_i)$ (e.g., end-point boundary cases) and this makes it more challenging to parallelize the computation. 

We would like to impose even functional regularization at every $f(x_i)$ based on how much the value deviates from the witness associated with the local region $\mathcal{B}(x_i)$. This approach leads to an \emph{asymmetric} co-operative formulation, where the information sets for the predictor $f$ and local witnesses $g_{x_i}$ differ. Specifically, the local best-response witness $\hat g_{x_i}$ is chosen to minimize the local discrepancy as in Eq.~(\ref{eq:local_deviation}), and thus depends on $f$ values within the whole region; in contrast, the predictor $f$ only receives feedback in terms of the resulting deviation at $x_i$, only seeing $\hat g_{x_i}(x_i)$. From the point of view of the predictor $f$, the best response strategy is obtained by minimizing 
\begin{align}
& \sum_{(x_i, y_i) \in \mathcal{D}} \bigg[\mathcal{L}(f(x_i), y_i) + \lambda \; d(f(x_i), \hat g_{x_i}(x_i))\bigg]
\label{eq:asym_coop_game}
\end{align}
To train the proposed method, we perform alternating updates for $f(\cdot)$ and $\hat{g}_{x_i}(\cdot)$ on their respective criteria.
% Note that the asymmetric game a reasonable criterion only when the neighborhoods are constructed from the labeled data: $\cup_{x_i \in \mathcal{D}_x} \mathcal{B}(x_i) \subseteq \mathcal{D}_x$. Otherwise, the $f$ values outside the labeled data are uncontrolled. This phenomenon does not occur in the symmetric game formulation precisely because of the symmetry of its criterion.

\subsection{Analysis}\label{seq:analysis}

We consider here the effectiveness of regularization in relation to the neighborhood size and establish fixed point equations for the predictor under the three estimation criteria. For simplicity, we assume $\mathcal{X} = \mathbb{R}^d$ and $\mathcal{Y} = \mathbb{R}$, but the results are generalizable to our examples in \xref{sec:prob}. All the proofs are in Appendix~\ref{sec:proof}.

\header{Neighborhood size.} The formulation involves a key trade-off between the size of the region where the function should be simple and the overall accuracy achieved by the predictor. When the neighborhood is too small, local witnesses become perfect, inducing no regularization on $f$. Thus the size of the region is a key parameter. A neighborhood size is sufficient if the witness class $\mathcal{G}$ cannot readily overfit $f$ values within the neighborhood. Formally, 
\iffalse
\begin{definition}
%\vspace{-2mm}
We say that a neighborhood size $m$ is \textbf{effective} for $\mathcal{G}$ if for any $f\not\in \mathcal{G}$, there exists a neighborhood $\mathcal{B} \subset \mathcal{X}: |\mathcal{B}| = m$ such that
%\footnote{The definition is established from the perspective of the witness. A similar definition can be derived from the perspective of the predictor $f$, such as ``$\forall f \notin \mathcal{G}, \exists \mathcal{B} ...$'', but we found for linear models and decision trees, \textred{the sample complexities are the same under the both definitions.}}
\begin{align}
    \min_{g \in \mathcal{G}} \frac{1}{m} \sum_{x \in \mathcal{B}} d(f(x), g(x)) > 0.
\end{align}
\vspace{-3mm}
\end{definition}
\vspace{-2mm}
\fi 
\begin{definition}
\label{def:effective}
%\vspace{-2mm}
We say that a neighborhood size $m$ is \textbf{effective} for $\mathcal{G}$ if for any $f\not\in \mathcal{G}$ we can find $\mathcal{B} \subset \mathcal{X}: |\mathcal{B}| = m$ s.t.
\begin{align}
    \min_{g \in \mathcal{G}} \frac{1}{m} \sum_{x \in \mathcal{B}} d(f(x), g(x)) > 0.
\end{align}
\vspace{-3mm}
\end{definition}
\vspace{-2mm}
A trivial example is when $\mathcal{G}$ is the constant class, a neighborhood size $m$ is effective if $m > 1$.
Note that the neighborhood $\mathcal{B}$ in the above definition can be any finite collection of points $\mathcal{B}(\cdot)$. For example, the points in the neighborhood induced by a temporal window $\{x_{i-\epsilon}, \dots, x_{i+\epsilon}\}$ need not remain in a small $\ell_p$-norm ball. 

For linear models and decision trees, we have 
\begin{itemize}[leftmargin=5.mm]
\vspace{-3.5mm}
\setlength\itemsep{0.2em}
  \item $d+1$ is the tight lower bound on the effective neighborhood size for the linear class. 
  \vspace{-1mm}
  \item $2^k+ 1$ is the tight lower bound on the effective neighborhood size for the decision tree class with depth bounded by $k$. 
\vspace{-2.5mm}
\end{itemize}
% The derivation is in Appendix~\ref{sec:proof}.
When the sample sizes within the neighborhoods fall below such bounds, regularization can still be useful if the witness class is not uniformly flexible or if the algorithm for finding the witness is limited (e.g., greedy algorithm for decision trees). 

\header{Equilibrium solutions.} 
The symmetric game constitutes a standard minimization problem, but the existence or uniqueness of equilibria under the asymmetric game are not obvious. Our main results in this section make the following assumptions.

\begin{itemize}[label={}, leftmargin=-0.5mm]
\vspace{-3.5mm}
\setlength\itemsep{-0.02em}
  \item \header{(A1)} the predictor $f$ is unconstrained.
  \vspace{-1mm}
  \item \header{(A2)} both the loss and deviation are squared errors.
  \vspace{-1mm}
  \item \header{(A3)} $|\mathcal{B}(x_i)| = m, \forall x_i \in \mathcal{D}_x$.
  \vspace{-1mm}
  \item \header{(A4)} $x_j \in \mathcal{B}(x_i) \implies x_i \in \mathcal{B}(x_j), \forall x_i, x_j \in \mathcal{D}_x$.
  \vspace{-1mm}
  \item \header{(A5)} $\cup_{x_i \in \mathcal{D}_x} \mathcal{B}(x_i) = \mathcal{D}_x$.
\vspace{-2.5mm}
\end{itemize}
We note that \textbf{(A3)} and \textbf{(A4)} are not technically necessary but simplify the presentation. We denote the predictor in the uniform criterion (Eq.~\eqref{eq:uniform:game}), the symmetric game (Eq.~\eqref{eq:coop_game}), and the asymmetric game (Eq.~\eqref{eq:asym_coop_game}) as $f_U$, $f_S$, and $f_A$, respectively. 
We use $X_i\in \mathbb{R}^{m\times d}$ to denote the neighborhood $\mathcal{B}(x_i) = \{x'_1,\dots,x'_m\}$ ($X_i = [x'_1,\dots,x'_m]^\top$), and $f(X_i)\in \mathbb{R}^m$ to denote the vector $[f(x'_1),\dots,f(x'_m)]^\top$. $X_j^\dagger$ denotes the pseudo-inverse of $X_j$. Then we have
\begin{theorem}
\label{lemma:linear_case}
If \emph{\textbf{(A1-5)}} hold and the witness is in the linear family, %Eq. (\ref{eq:opt_s_const}) is the unique equilibrium for $f_S$
the optimal $f_S$ satisfies 
\vspace{-2mm}
\begin{equation}
f^*_S(x_i) = \frac{1}{1+\lambda} \bigg[ y_i + \frac{\lambda}{m} \bigl(\sum_{x_j \in \mathcal{B}(x_i)} X^\dagger_j f^*_S(X_j) \bigr)^\top x_i \bigg], \nonumber
%\label{eq:opt_s_const}
\vspace{-2mm}
\end{equation}
and the optimal $f_A$, at every equilibrium, is the fixed point
\begin{equation}
\vspace{-2mm}
f^*_A(x_i) = \frac{1}{1+\lambda} \bigg[ y_i + \lambda  (X^\dagger_if^*_A(X_i) )^\top  x_i \bigg], \forall x_i \in \mathcal{D}_x. \nonumber %\label{eq:opt_a_const}%\nonumber
% \vspace{-2mm}
\end{equation}
\end{theorem}
The equilibrium in the linear class is not unique when the witness is not fully determined in a neighborhood due to degeneracy. To avoid these cases, we can use Ridge regression to obtain a stable equilibrium (proved also in Appendix). 

%As a linear explanation is frequently adopted for interpretability, the existence of its equilibrium is useful. 
A special case of Theorem~\ref{lemma:linear_case} is when $x_i = [1], \forall x_i \in \mathcal{D}_x$, which effectively yields the equilibrium result for the constant class; we found it particularly useful to understand the similarity between the two games in this scenario. Concretely, each $(X^\dagger_j f(X_j))^\top x_i$ becomes equivalent to $\frac{1}{m} \sum_{x_k \in \mathcal{B}(x_j)} f(x_k)$. As a result, the solution for both the symmetric and asymmetric game induce the optimal predictors as recursive convolutional averaging of neighboring points with the same decay rate ${\lambda}/({1+\lambda})$, while the convolutional kernel evolves twice as fast in the symmetric game than in the asymmetric game. 

% Finally, we analyze the equilibrium for the uniform criterion, which results in a very different equilibrium. 
Next, we show that the hard uniform constraint criterion yields a very different equilibrium. 
\begin{theorem}
If \emph{\textbf{(A1-5)}} hold and the witness is in the linear family, the optimal $f_U$ satisfies
\begin{align}
    f_U^*(x_i) = \left \{
  \begin{aligned}
    & \alpha(x_i, f_U^*), && \text{if}\ \alpha(x_i, f_U^*) > y_i, \\
    & \beta(x_i, f_U^*), && \text{if}\ \beta(x_i, f_U^*) < y_i, \\
    & y_i, && \text{otherwise,}
  \end{aligned} \right. \nonumber
\end{align}
for $x_i \in \mathcal{D}_x$, where
\begin{align*}
    & \alpha(x_i, f^*_U) = \max_{x_j \in \mathcal{B}(x_i)} \bigg[ (X_j^\dagger f^*_U(X_j))^\top x_i \\
    & - \sqrt{\delta m - \sum_{x_k \in \mathcal{B}(x_j)\backslash \{x_i\} }  (f^*_U(x_k) -  (X_j^\dagger f^*_U(X_j))^\top x_k )^2} \bigg];\\
    & \beta(x_i, f^*_U) = \min_{x_j \in \mathcal{B}(x_i)} \bigg[ (X_j^\dagger f^*_U(X_j))^\top x_i \\
    & + \sqrt{\delta m - \sum_{x_k \in \mathcal{B}(x_j)\backslash \{x_i\} }  (f^*_U(x_k) -  (X_j^\dagger f^*_U(X_j))^\top x_k )^2} \bigg].
\end{align*}
% \label{theorem:appendix:uniform}
\vspace{-4.5mm}
\end{theorem}
A noticeable difference from the games is that, under uniform criterion, the optimal predictor $f^*_U(x_i)$ may faithfully output the actual label $y_i$ if the functional constraint is satisfied, while the functional constraints are translated into a ``convolutional'' operator in the games. 

%In contrast, the constant case is easier to identify the qualitative similarity between these two forms of the game \emph{at equilibrium}: both games induce the optimal predictors as recursive convolutional averaging of neighboring points with the same decay rate ${\lambda}/({1+\lambda})$, while the convolutional kernel evolves twice faster in the symmetric game than in the asymmetric game. 
\header{Efficient computation.} We also analyze ways of accelerating the computation required for solving the symmetric game. An equivalent criterion is given by
\begin{lemma}\label{lemma:main_adjusted_game}
If $d(\cdot, \cdot)$ is squared error, $\mathcal{L}(\cdot, \cdot)$ is differentiable, $f$ is sub-differentiable, and \emph{$\textbf{A(4-5)}$} hold, then 
\vspace{-1.5mm}
\begin{equation}
\sum_{(x_i, y_i) \in \mathcal{D}}\! \mathcal{L}(f(x_i), y_i) + \frac{\lambda}{\bar{N}_i}\bigg[ \bar{N}_i f(x_i) - \!\!\!\! \sum_{x_t\in \mathcal{B}(x_i)} \! \frac{\hat{g}_{x_t}(x_i)}{|\mathcal{B}(x_t)|} \bigg]^2\!\!\!, \! \nonumber
\end{equation}
where $\bar{N}_i := \sum_{x_t \in \mathcal{B}(x_i)}\! \frac{1}{|\mathcal{B}(x_t)|}$, induces the same equilibrium as the symmetric game. 
\end{lemma}
The result is useful when training $f$ on GPU and $\hat{g}_{x_i}$ is solved analytically on CPU. 
Compared to a for-loop to handle different neighborhood sizes for Eq.~(\ref{eq:coop_game}) on the GPU, computing a summarized feedback as in Lemma~\ref{lemma:main_adjusted_game} on CPU is more efficient (and easier to implement). 

%Compare this to a for-loop to handle different neighborhood sizes for Eq.~(\ref{eq:coop_game}) on the GPU. Computing a summarized feedback as in Lemma~\ref{lemma:main_adjusted_game} on CPU is more efficient (and easier to implement). 

\header{Discussion} We investigated here discrete neighborhoods and they are suitable also for structured data as in the experiments. The method itself can be generalized to continuous neighborhoods with an additional difficulty: the exact computation and minimization of functional deviation between the predictor and the witness in such neighborhood is in general intractable. We may apply results from learning theory (e.g.,~\cite{shamir2015sample}) to bound the (generalization) gap between the deviation computed by finite samples from the continuous neighborhood and the actual deviation under a uniform probability measure. %We put another discussion with respect to GANs~\citep{goodfellow2014generative} in Appendix~\ref{app:discuss:game}.

\section{Examples}
\label{sec:prob}

\subsection{Conditional Sequence Generation}

The basic idea of co-operative modeling extends naturally to conditional sequence generation over longer periods. 
%For example, we may ask how a stock price evolves over a fixed horizon given historical prices. 
Broadly, the mechanism allows us to inspect the temporal progression of sequences on a longer term basis. 

Given an observation sequence $x_{1},\dots,$ $x_{t}\in\mathbb{R}^c$, the goal is to estimate probability $p(x_{t+1:T} | x_{1:t})$ over future events $x_{t+1},\dots,$ $x_{T}\in \mathbb{R}^c$, typically done via maximum likelihood. 
For brevity, we use $x_{1:i}$ to denote $x_1,\dots,x_i$.
We model the conditional distribution of $x_{i+1}$ given ${x}_{1:i}$ as a multivariate Gaussian distribution with mean $\mu(x_{1:i})$ and covariance $\Sigma(x_{1:i})$, both parametrized as recurrent neural networks.
Each local witness model $g_{{x}_{1:i}}(\cdot)$ is estimated based on the neighborhood $\mathcal{B}(x_{1:i}):=\{{x}_{1:i-\epsilon},\dots,{x}_{1:i+\epsilon}\}$ with respect to the mean function $\mu(\cdot)$.   
%$\mathcal{B}(x_{1:i}):=\{({x}_{1:i-\epsilon}, \mu({x}_{1:i-\epsilon})),$ $\dots,({x}_{1:i+\epsilon}, \mu({x}_{1:i+\epsilon}))\}$. 
A natural choice would be a $K$-order Markov autoregressive (AR) model with an $\ell_2$ deviation loss as:
% \vspace{-1mm}
\begin{equation}
\vspace{-1mm}
    \min_{\theta} \sum_{x_{1:t} \in \mathcal{B}(x_{1:i})} \| \sum_{k=0}^{K-1} \theta_{k+1} \cdot x_{t-k} + \theta_0 - \mu(x_{1:t}) \|^2_2, \nonumber
\end{equation}
where $\theta_k \in \mathbb{R}^{c \times c}, \forall k > 0$ and $\theta_0 \in \mathbb{R}^{c}$.
The AR model admits an analytical solution similar to linear regression. 
% AR model is a generalization of linear models to temporal data and thus also admits an analytical solution. 
% Once the predictor is estimated under the guiding influence of the local witnesses, the actual trajectory can be unfolded in a feed-forward fashion using the predictor alone, or with feedback as the trajectory extension that minimizes Eq. (\ref{eq:implicit}).

\subsection{Chemical Property Prediction}
\label{sec:chemprop}
The models discussed in \xref{sec:method} can be instantiated on highly-structured data, such as molecules, too. These are usually represented as a graph $\mathcal{M}=(\mathcal{V}, \mathcal{E})$ whose nodes encode the atom types and edges encode the chemical bonds.
Such representation enables the usage of recent graph convolutional networks (GCNs)~\citep{dai2016discriminative,lei2017deriving} as the predictor $f$. As it is hard to realize a simple explanation on the raw graph representation, we exploit an alternative data representation for the witness model; we leverage depth-bounded decision trees that take as input Morgan fingerprints~\cite{rogers2010extended} $x(\mathcal{M})$, which are vector representations for the binary existence of a chemical substructures in a molecule (e.g., the nodes in Fig.~\ref{fig:teaser}).

% The neighborhood $\mathcal{B}(\mathcal{M})$ includes molecules $\{\mathcal{M}'\}$ with Tanimoto similarity greater than threshold $\epsilon=0.6$, automatically constructed through matching molecular pair analysis~\citep{griffen2011matched}.
The neighborhood $\mathcal{B}(\mathcal{M})$ includes molecules $\{\mathcal{M}'\}$ with Tanimoto similarity greater than $0.6$, automatically constructed through matching molecular pair analysis~\citep{griffen2011matched}.
Here we use a multi-label binary classification task as an example, and adopt a cross-entropy loss for each label axis for simplicity. At each neighborhood $\mathcal{B}(\mathcal{M})$, we construct a witness decision tree $g$ that minimizes the total variation (TV) from the predictor as
%If we denote $x(\mathcal{V}, \mathcal{E})$ as the Morgan fingerprint of the molecule $(\mathcal{V}, \mathcal{E})$, 
\vspace{-1.5mm}
\begin{equation}
    \min_{g\in \mathcal{G}_{tree}} \!\! \frac{1}{|\mathcal{B}(\mathcal{M})|} \sum_{\mathcal{M}' \in \mathcal{B}(\mathcal{M}) } \!\!\! \sum_{i=1}^{\text{dim}(\mathcal{Y})} \!\! |f(\mathcal{M}')_i - g(x(\mathcal{M}'))_i|. \!\! \label{eq:dctree} %\nonumber
\end{equation}
We note that Eq.~(\ref{eq:dctree}) is an upper bound and efficient alternative to fitting a tree for each label axis independently. 

\subsection{Molecule Representation Learning}
\label{sec:molrepr}

Our approach can be further applied to learn transparent latent graph representations by variational autoencoders (VAEs)~\citep{kingma2013auto,jin2018junction}. Concretely, given a molecular graph $\mathcal{M}=(\mathcal{V}, \mathcal{E})$, the VAE encoder $q$ outputs the approximated posterior $z_\mathcal{M} \sim \mathcal{N}(\mu_\mathcal{M}, \Sigma_\mathcal{M})$ over the latent space, where $z_\mathcal{M}$ is the continuous representation of molecule $\mathcal{M}$. 
Following common practice, $\Sigma_\mathcal{M}$ is restricted to be diagonal. 
The VAE decoder then reconstructs the molecule $\mathcal{M}$ from its probabilistic encoding $z_\mathcal{M}$. 
Our goal here is to guide the behavior of the neural encoder $q$ such that the derivation of (probabilistic) $z_\mathcal{M}$ can be locally explained by a decision tree.
% learn a witness function $g$ that faithfully explains the derivation of (probabilistic) $z_\mathcal{M}$.
% Our goal here is to learn a witness function $g$ that faithfully explains the derivation of (probabilistic) $z_\mathcal{M}$.

We adopt the same setting for the witness function and neighborhoods as in \xref{sec:chemprop}, except that the local decision tree $g$ now outputs a joint normal distribution with parameters $[\widehat{\mu}_\mathcal{M}, \widehat{\Sigma}_\mathcal{M}]$.
To train the encoder, we extend the original VAE objective $\mathcal{L}^{\mathrm{VAE}}$ with a local deviation loss $\mathcal{L}^{\mathcal{G}_\text{tree}}$ defined on the KL divergence between the VAE posterior $q(\mathcal{M})=\mathcal{N}(\mu_\mathcal{M}, \Sigma_\mathcal{M})$ and witness posterior $g(x(\mathcal{M})) = \mathcal{N}(\widehat{\mu}_\mathcal{M}, \widehat{\Sigma}_\mathcal{M})$ at each neighborhood as
\vspace{-1mm}
\begin{equation}
% \mathcal{L}^{\mathcal{G}_\text{tree}} \! := \frac{1}{|\mathcal{D}|}\sum_{\mathcal{M} \in \mathcal{D}} \min_{g \in \mathcal{G}_\text{tree}} \sum_{\mathcal{M}' \in \mathcal{B}(\mathcal{M}) } \mathrm{KL} \! \left(g(x(\mathcal{M}')) || q(\mathcal{M}')\right) \nonumber
\mathcal{L}^{\mathcal{G}_\text{tree}} \! := \frac{1}{|\mathcal{D}|}\sum_{\mathcal{M} \in \mathcal{D}} \min_{g \in \mathcal{G}_\text{tree}} \sum_{\mathcal{M}' \in \mathcal{B}(\mathcal{M}) } \frac{\mathrm{KL} \! \left(g(x(\mathcal{M}')) || q(\mathcal{M}')\right)}{|\mathcal{B}(\mathcal{M})|} \nonumber
\vspace{-1mm}
\end{equation}
The VAE is trained to maximize $\mathcal{L}^{\mathrm{VAE}} + \lambda \cdot \mathcal{L}^{\mathcal{G}_\text{tree}}$. For ease of implementation, we asymmetrically estimate each decision tree $g$ with mean squared error between the vectors $[\mu_\mathcal{M}, \Sigma_\mathcal{M}]$ and $[\widehat{\mu}_\mathcal{M}, \widehat{\Sigma}_\mathcal{M}]$.

\vspace{-1mm}

\vspace{-1mm}
\section{Experiments}
\label{sec:experiment}
\vspace{-1mm}

We conduct experiments on chemical and time-series datasets. Due to the lack of existing works for explaining structured data, we adopt an ablation setting -- comparing our approach (\textsc{Game}) versus an unregularized model (\textsc{Deep}) -- and focus on measuring the transparency.
We use subscripts to denote specific versions of the \textsc{Game} models.
%Since the local witnesses also constitute a valid function of every $x_i \in \mathcal{X}$ as $\hat{g}_{x_i} (x_i)$, we can also validate its performance against the labels.
Note that we only fit the local witnesses to the \textsc{Deep} model during testing for evaluation. 
Unless otherwise noted, the reported results are based on the testing set. 

%We evaluate our approach on applications to time series generation and language modeling with quantitative and qualitative experiments, highlighting transparency -- in terms of deviation from a functional class -- and performance. Extensive experiments are conducted to analyze the relation between hyper-parameters and the induced functions. Unless otherwise stated, all the results reported correspond to the testing set. 
\vspace{-0.5mm}
\vspace{-1mm}
\subsection{Molecule Property Prediction}
\vspace{-1mm}
\begin{table}[t]
\vspace{-3mm}
\setlength{\tabcolsep}{3.9pt}
\caption{Performance on the Tox-21 dataset. $\text{AUC}_{\mathcal{D}}(\hat{g}_{\mathcal{M}}, f)$ and $\text{AUC}_{\mathcal{B}}(\hat{g}_{\mathcal{M}}, f)$ generalize the AUC score to use $f$ values as labels, computed on the testing data and their neighborhoods, respectively.}
\vspace{3pt}
\label{tab:tox_perf}
\centering
\begin{tabular}{ c | c c c c  c c c}  
\toprule
\small Aspect & \small Measure    & \small $\textsc{Game}_\text{unif}$                         & \small $\textsc{Game}_\text{sym}$  & \small \textsc{Deep}\\
\midrule
%\multirow{2}{*}{\small Performance} & \small predictor AUC   & \small 0.826 & 0.815\\
{\small Performance}            & \small AUC$(f, y)$                                             & \small 0.744 & \small \textbf{0.826} & \small 0.815\\
\tiny (the higher the better)   & \small AUC$(\hat{g}_{\mathcal{M}}, y)$                         & \small 0.742 & \small \textbf{0.824} & \small 0.818\\
\midrule
{\small Transparency}           & \small $\text{AUC}_{\mathcal{B}}(\hat{g}_{\mathcal{M}}, f)$    & \small \textbf{0.764} & \small 0.759 & \small 0.735\\
\tiny (the higher the better)   & \small $\text{AUC}_{\mathcal{D}}(\hat{g}_{\mathcal{M}}, f)$    & \small 0.959 & \small \textbf{0.967} & \small 0.922\\
% \midrule
% {\small Interpretability} & \small tree depth      & \small 7.1   & 7.9\\
% \small (the smaller the better) & \small node count      & \small 45.7  & 64.4\\
\bottomrule
\end{tabular}
\vspace{-3mm}
\end{table}	

%We conduct experiments on molecular toxicity prediction on the public Tox21\footnote{https://tripod.nih.gov/tox21/challenge/} dataset, which contains 12 binary labels and $7,831$ molecules. The labels are very unbalanced; the fraction of the positive label is between $16.15\%$ and $3.51\%$ among the 12 labels. 
% We use GCNs as the predictor and decision trees as the witness with the symmetric game.
% with the uniform criterion ($\textsc{Game}_\text{unif.}$) and the symmetric game ($\textsc{Game}_\text{sym.}$).
We conduct experiments on molecular toxicity prediction on the Tox21 dataset from MoleculeNet benchmark~\citep{wu2018moleculenet}, which contains 12 binary labels and $7,831$ molecules. The labels are very unbalanced; the fraction of the positive label is between $16.15\%$ and $3.51\%$ among the 12 labels.
We use GCN as the predictor and decision trees as the witnesses as in \xref{sec:chemprop}. 
% Here  The
The neighborhood sizes $m$ of about $60\%$ of the molecules are larger than $2$, whose median and maximum are $59$ and $300$, respectively.
Since each neighborhood has a different size $m$, we set the maximum tree depth as $\max\{\lceil \log_2(m) \rceil - 1, 1\}$ for each neighborhood, which ensures that the corresponding size $m$ is effective for $m > 2$ (see Definition~\ref{def:effective}). 
More details are in Appendix~\ref{appendix:molecule}.
 
\header{Evaluation Measures:} 
For all the measures, the results are averaged across the label axes. 
\vspace{-1mm}

(1) Performance: For the predictor, we compare its predictions with respect to the labels in AUC, denoted as AUC$(f, y)$. As each local witness $\hat{g}_{\mathcal{M}}(x(\mathcal{M}))$ also realizes a function of $\mathcal{M}$, it is also evaluated against the labels in AUC, denoted as AUC$(\hat{g}_{\mathcal{M}}, y)$.
\vspace{-1mm}

(2) Transparency: As labels are unavailable for testing data in practice, it is more realistic to measure the similarity between the predictor and the local witnesses to understand the validity of the explanations derived from the decision trees $\mathcal{G}$. To this end\footnote{Since the predictor probability can be scaled arbitrarily to minimize the TV from decision trees without affecting performance, using TV to measure transparency as used in training is not ideal.}, we generalize the AUC criterion for continuous labels for $N$ references $y$ and predictions $y'$ as
\vspace{-1mm}
\begin{equation}
\sum_{i=1}^{N} \sum_{j=1}^{N} \mathbb{I}(y_i > y_j) \mathbb{I}(y'_i > y'_j) / \sum_{i=1}^{N} \sum_{j=1}^{N} \mathbb{I}(y_i > y_j). \nonumber
\vspace{-1mm}
\end{equation}
The proposed score has the same pairwise interpretation as AUC, recovers AUC when $y$ is binary, and is normalized to $[0, 1]$. 
%We use the measure for the local witnesses with respect to the predictor (the reference), computed on the testing data and their neighborhoods, which are denoted as 
Locally, we measure the criterion for the local witnesses with respect to the predictor in each testing neighborhood as the local deviation, where the average result is denoted as $\text{AUC}_{\mathcal{B}}(\hat{g}_{\mathcal{M}}, f)$.
Globally, the criterion is also validated among the testing data, denoted as $\text{AUC}_{\mathcal{D}}(\hat{g}_{\mathcal{M}}, f)$. 
% We compute the core for the local witnesses with respect to the predictor  (the reference) on the testing data as the global deviation, denoted as $\text{AUC}_{\mathcal{D}}(g_{\mathcal{M}}, f)$. 
% For each local neighborhood of a testing molecule $\mathcal{M}$, we also measure the score between the local witness and the predictor within the neighborhood  as the local deviation 

The results with the uniform and symmetric criteria are shown in Table~\ref{tab:tox_perf}. A baseline vanilla decision tree, with depth tuned between $2$ and $30$, yields 0.617 in $\text{AUC}(f, y)$.
Compared to $\textsc{Game}_\text{sym}$, the local deviation in $\textsc{Game}_\text{unif}$ is marginally improved due to the strict constraint at the cost of severe performance loss.
We investigate the behaviors in training neighborhoods and find that $\textsc{Game}_\text{sym}$ exhibits a tiny fraction of high deviation losses, allowing the model to behave more flexibly than the strictly constrained $\textsc{Game}_\text{unif}$ (see Figure~\ref{fig:train:tv} in Appendix~\ref{appendix:molecule}).
In terms of performance, our $\textsc{Game}_\text{sym}$ model is superior to the \textsc{Deep} model in both the predictor and local witnesses. 
When comparing the witnesses to the predictor, locally and globally, the \textsc{Game} models significantly improve the transparency from the \textsc{Deep} model. 
The local deviation should be interpreted relatively since the tree depth inherently prevents local overfitting. 
% When assessed globally, the deviation is magnified, resulting in a significant improvement from $0.922$ to $0.967$ for representing the predictor as decision trees. 
%In terms of tree depth and node counts, the proposed \textsc{Game} model also improves the interpretability for explaining a GCN model. %In summary, the proposed approach improves vanilla neural on chemical property prediction for better performance, transparency, and interpretability.

We visualize the resulting witness trees in Figure~\ref{fig:teaser} under the same transparency constraint: for a local neighborhood, we grow the witness tree for the \textsc{Deep} model until the local transparency in $\text{AUC}_{\mathcal{B}}$ is comparable to the $\textsc{Game}_\text{sym}$ model.
For explaining the same molecule, the tree for the \textsc{Deep} model is deeper and extremely unbalanced. Since a Morgan fingerprint encodes the existence of a substructure of molecule graphs, an unbalanced tree focusing on the left branch (non-existence of a substructure) does not capture much generality. Hence, the explanation of the \textsc{Deep} model does not provide as much insight as our $\textsc{Game}_\text{sym}$ model.
% \ghl{(If I have time) Figure~\ref{fig:teaser} confers to the testing molecule that yields the $20^\text{th}$ percentile in the node counts on the witness tree. The $40^\text{th}$ percentile is visualized in Appendix X.}

Here we do an analysis on the tree depth constraint for the witness model, as a shallower tree is easier to interpret, but more challenging to establish transparency due to the restricted complexity. 
To this end, we revise the depth constraint to $\max\{\lceil \log_2(m) \rceil - 1 + \Delta, 1\}$ during training and testing, and vary $\Delta \in \{-3,\dots,0\}$. 
All the resulting \textsc{Game} models outperform the \textsc{Deep} models in AUC$(f, y)$, and we report the transparency score in terms of $\text{AUC}_\mathcal{D}(\hat{g}_{\mathcal{M}}, f)$ in Table~\ref{tab:tox_anal}. Even when $\Delta = -3$, the witness trees in our \textsc{Game} model still represent the predictor more faithfully than those in the \textsc{Deep} model with $\Delta = 0$.

\begin{table}[t]
\vspace{-2.5mm}
\setlength{\tabcolsep}{3.9pt}
\caption{$\text{AUC}_\mathcal{D}(\hat{g}_{\mathcal{M}}, f)$ score on different $\Delta$ in the Tox-21 dataset (lower $\Delta$ implies shallower trees).}
\vspace{3pt}
\label{tab:tox_anal}
\centering
\begin{tabular}{ c | c c c c  c c c}  
\toprule
\small Model & \small $\Delta=0$ & \small $\Delta=-1$ & \small $\Delta=-2$ & \small $\Delta=-3$\\
\midrule
\small \textsc{Game} & \small 0.967 & \small 0.967 & \small 0.964 & \small 0.958\\
\small \textsc{Deep} & \small 0.922 & \small 0.916 & \small 0.915 & \small 0.914\\
\bottomrule
\end{tabular}
\vspace{-5mm}
\end{table}

%(3) Interpretability: since the tree depth and node counts reflect the effort to interpret a decision tree, we use them as a surrogate measure for interpretability. For fair comparison, evaluation is conducted under the same level of transparency (local-GAUC). We first observe that for a local decision tree, it can hardly achieve $1.0$ GAUC in the neighborhood since our tree depth prevents the local witness from overfitting. As a result, we adopt a relative criterion; for each neighborhood, we use the maximum GAUC between the \textsc{Game} and the \textsc{Deep} models given the original depth constraint as the baseline, and expand the tree depth for the other witness until it achieves equal or higher GAUC. Finally, we report the average tree depth and node count of the resulting witnesses across the testing neighborhoods. 

\vspace{-0.5mm}
\subsection{Physical Component Modeling}

% \begin{table}[t]
% \setlength{\tabcolsep}{3.9pt}
% \caption{Performance on the bearing dataset.}
% \label{tab:peref}
% \centering
% \begin{tabular}{ c | c c c  c c c}  
% \toprule
% $(\times 10^{-2})$ & \small Error & \small Deviation & \small TV\\
% \midrule
% \small AR   & \small 9.832 & \small 0.000 & \small 0.000\\
% \midrule
% \small \textsc{Game} & \small 8.309 & \small 3.431 & \small 5.706\\
% \small \textsc{Deep} & \small 8.136 & \small 4.197 & \small 7.341\\
% \bottomrule
% \end{tabular}
% \vspace{-2.5mm}
% \end{table}	

\begin{figure*}[t]
\centering
  \includegraphics[width=.9\linewidth]{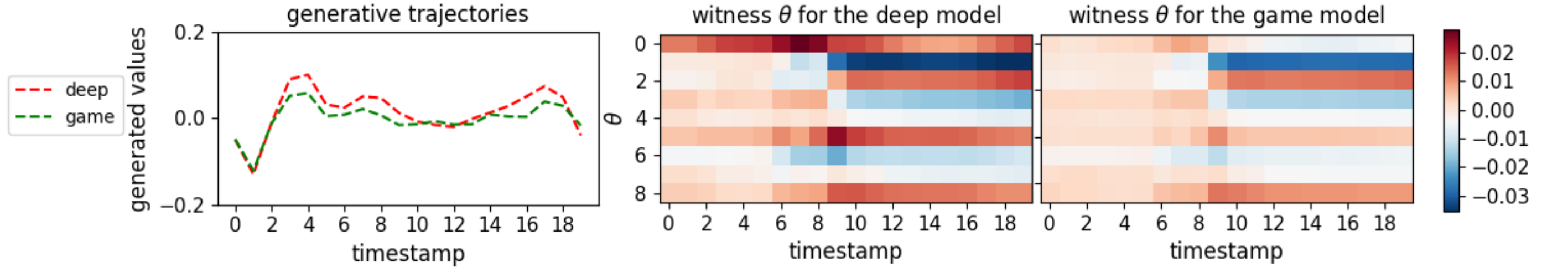}
  \vspace{-5mm}
  \caption{Visualizing the linear witnesses (middle and right plots) on the first channel (left plot) along the autoregressive generative trajectory ($x$-axis) on the bearing dataset. The $y$-axis of the parameters from $0$ to $8$ denotes the bias $(\theta_0)_1$ and weights $(\theta_1)_{1, 1:4}, (\theta_2)_{1, 1:4}$.}
  \vspace{-4mm}
  \label{fig:visual}
\end{figure*}

\begin{table}[t]
\setlength{\tabcolsep}{3.2pt}
\caption{Performance of the symmetric and asymmetric setting of the \textsc{Game} model with $\epsilon=9$.}
\vspace{3pt}
%The game model (\small{$\epsilon=9$}) on the bearing dataset.} % when $\lambda$ varies. $\epsilon$ is fixed to \small{$9$}.}
\label{tab:lbda}
\centering
\begin{tabular}{c  c || c  c  c  c  c | c }  
\toprule
\small $(\times 10^{-2})$ & \bf \small {$\lambda$} & \bf\small $0$ & \bf\small $0.1$ & \small \bf $1$ & \small \bf $10$ & \small \bf $100$ & AR\\
\midrule
\multirow{2}{*}{\small $\textsc{Game}_\text{asym}$} & \small Error & \small 8.136 & \small 8.057 & \small 8.309 & \small 9.284 & \small 9.794 & \small 9.832\\
 & \small Dev.        						        & \small 4.197 & \small 4.178 & \small 3.431 & \small 1.127 & \small 0.186 & \small0.000\\
 & TV                                               & \small 7.341 & \small 7.197 & \small 5.706 & \small 1.177 & \small 0.144 & \small 0.000\\
\midrule
\multirow{2}{*}{\small $\textsc{Game}_\text{sym}$} & \small Error  & \small 8.136 & \small 8.089 & \small 8.315 & \small 9.314 & \small 9.807 & \small 9.832\\
 & \small Dev.        						        & \small 4.197 & \small 4.169 & \small 3.426 & \small 1.116 & \small 0.182 & \small 0.000\\
 & TV                                               & \small 7.341 & \small 7.292 & \small 5.621 & \small 1.068 & \small 0.132 & \small 0.000\\
\bottomrule
\end{tabular}
\vspace{-4.5mm}
\end{table}	

\iffalse
\begin{table*}[thb]
\caption{The \textsc{Game} models on bearing dataset with different $\epsilon$. The asymmetric game is used with $\lambda = 1$.}
\label{tab:eps}
\centering
\begin{tabular}{l | c  c  c  c  c  c  c  c }  
\toprule
\bf {} & \bf\small $\epsilon=1$ & \bf\small $\epsilon=5$ & \bf\small $\epsilon=7$ & \small \bf $\epsilon=9$ & \small \bf $\epsilon=13$ & \small \bf $\epsilon=17$ & \small \bf $\epsilon=19$\\
\midrule
Error            & 0.08225 & 0.08302 & 0.08316 & 0.08309 & 0.08349 & 0.08345 & 0.08344\\
Deviation        & 0.01697 & 0.02762 & 0.03134 & 0.03431 & 0.03623 & 0.03614 & 0.03639\\
TV               & 0.09185 & 0.07489 & 0.06644 & 0.05706 & 0.03244 & 0.01163 & 0.00000\\
\bottomrule
\end{tabular}
\end{table*}	
\fi

We next validate our approach on a physical component modeling task with the bearing dataset from NASA~\citep{bearing_dataset}, which records 4-channel acceleration data on 4 co-located bearings.
We divide the sequence into disjoint subsequences, resulting in $200,736$ subsequences.
Since the dataset exhibits high frequency periods of 5 points and low frequency periods of 20 points, we use the first $80$ points in an sequence to forecast the next $20$. 
We parametrize $\mu(\cdot)$ and $\Lambda(\cdot)$ jointly by stacking $1$ layer of CNN, LSTM, and $2$ fully connected layers.
We set the neighborhood radius $\epsilon$ to $9$ such that the witnesses are fit with completely different data for the beginning and the end of the sequence. The Markov order $K$ is set to $2$ to ensure the effectiveness of the neighborhood sizes. 
More details are in Appendix~\ref{app:time_series}.
%\ghl{We randomly sample $85\%$, $5\%$, and $10\%$ of the data for training, validation, and testing. We set the neighborhood radius $\epsilon$ and Markov order $K$ to and $9$ and $2$ to enable complete variation of functional properties. That is, the witnesses are fit with completely different data for the begginging and the end of the sequence.}

Evaluation involves three different types of errors: 1) `error' is the root mean squared error (RMSE) between greedy autoregressive generation and the ground truth, 2) `deviation' is RMSE between the predictor $\mu(x_{1:i})$ and the witness $\hat{g}_{x_{1:i}}(x_{1:i})$, and 3) `TV' is the average total variation of witness $\hat{g}_{x_{1:i}}$ parameters $[{\theta}, {\theta}_0]$ between every two consecutive time points. 
Since the deviation and error are both computed on the same space in RMSE, the two measures are readily comparable.
For testing, the witnesses are estimated based on the autoregressive generative trajectories. 

We present the results in Table~\ref{tab:lbda} to study the impact of the game coefficient $\lambda$ and the symmetry of the games. 
The trends in the measures are quite monotonic on $\lambda$: with an increasing $\lambda$, the model gradually operates toward the AR family with lower deviation and TV but higher error. 
When $\lambda = 0.1$, the \textsc{Game} models are more accurate than the \textsc{Deep} model ($\lambda = 0$) due to the regularization effect. 
% The discussion holds for both the symmetric and asymmetric games, and 
% The trends in TV are similar to the trends in deviation.
Given the same hyper-parameters, marginally lower deviation in the symmetric game than in the asymmetric game confirms our analysis about the similarity between the two. 
% Hence, we recommend users to use the asymmetric version since it is more efficient and substantially easier to implement than the symmetric game.
In practice, the asymmetric game is more efficient and substantially easier to implement than the symmetric game. 
Indeed, the training time is $20.6$ sequences/second for the asymmetric game, and $14.6$ sequences/second for the symmetric game. 
If we use the formula in Lemma~\ref{lemma:main_adjusted_game}, the symmetric game can be accelerated to $20.4$ sequences/second, but the formula does not generalize to other deviation losses.

We visualize the witnesses with their parameters $[\theta_0, \theta]$ along the autoregressive generative trajectories in Figure~\ref{fig:visual}. 
%The ground truth and the teacher-forced generative trajectories are available in Appendix~\ref{app:time_series}.
The stable functional patterns of the \textsc{Game} model as reflected by $\theta$, before and after the $9^\text{th}$ point, highlight not only close local alignments of the predictor and the AR family (being constant vectors across columns) but also flexible variation of functional properties on the predictor across regions.
In contrast, the \textsc{Deep} model yields unstable linear coefficients, and relies more on offsets/biases $\theta_0$ than the \textsc{Game} model, while the linear weights are more useful for grounding the coordinate relevance for interpretability.
Finally, we remark that despite the uninterpretable nature of temporal signals, the functional pattern reflected by the linear weights as shown here yields a simple medium to understand its behavior.
Due to space limitation, the additional analysis and visualization are included in Appendix~\ref{app:time_series}.

% \header{Analysis.} We study the impact of the game coefficient $\lambda$ and the asymmetry. 
% The results of the asymmetric game versus the symmetric game with different $\lambda$ are shown in Table~\ref{tab:lbda}. 
%TV is omitted due to the similar trend to deviation. 

% \iffalse
% the neighborhood analysis

% \fi
\vspace{-0.5mm}
\vspace{-1mm}
\subsection{Molecule Representation Learning}
\vspace{-1mm}
% original version
% We further use our method to interpret the latent space of a specific molecular generative model called junction tree variational autoencoder~\citep{jin2018junction}, which we denote as the \textsc{Deep} model. The corresponding \textsc{Game} model is trained with regularization described in section~\ref{sec:molrepr}. Both models are trained on the ZINC dataset~\citep{sterling2015zinc} containing 1.5M molecules, and evaluated on a test set with 20K molecules. We compare both models in terms of the evidence lower bound (ELBO) over the test set. Specifically, we evaluate all models under two scenarios: the performance of both models using the original latent representation (denoted as original ELBO); and the performance of models using the latent embedding generated by the witness function (denoted as witness ELBO). In addition, we report the KL divergence between the witness posterior and the original VAE posterior (denoted as witness KL). 

% The result is shown in Table~\ref{tab:molvae}. Our \textsc{Game} model performs consistently better under all metrics. Figure~\ref{fig:molvae_tree} provides an example of how decision tree explains the local neighborhood of a molecule. We found the root substructure is the indicator that divides the local neighborhood most evenly among all substructures (it occurred 19 times among 35 molecules in the neighborhood). This shows that the latent representation captures meaningful information. We provide more details of this example in Appendix.

Finally, we validate our approach on learning representations for molecules with VAEs, where we use the junction tree VAE~\citep{jin2018junction} as an example.
Here the encoders of VAEs, with and without the guidance of local decision trees as in \xref{sec:molrepr}, are denoted as \textsc{Deep} and \textsc{Game}, respectively.
The models are trained on the ZINC dataset~\citep{sterling2015zinc} containing 1.5M molecules, and evaluated on a test set with 20K molecules. We measure the performance in terms of the evidence lower bound (ELBO) over the test set. Here we consider two scenarios: the ELBO using the raw latent representations from the original neural encoder, and using the interpreted latent representations generated by locally fitted decision trees. 
The average deviation loss in KL divergence $\mathcal{L}^{\mathcal{G}_\text{tree}}$, defined in \xref{sec:molrepr}, over the testing neighborhoods is also evaluated. %maybe better to say "... over the test set is also reported."

The results are shown in Table~\ref{tab:molvae}. Our \textsc{Game} model performs consistently better under all the metrics.
%Figure~\ref{fig:molvae_tree} provides an example of how decision tree explains the local neighborhood of a molecule. We found the root substructure is the indicator that divides the local neighborhood most evenly among all substructures (it occurred 19 times among 35 molecules in the neighborhood). This shows that the latent representation does encode key information that help distinguish similar molecules.
Figure~\ref{fig:molvae_tree} shows an example of how our decision tree explains the local neighborhood of a molecule. We found most of the substructures selected by the decision tree occur in the side chains outside of Bemis-Murcko scaffold~\citep{bemis1996properties}. This shows the variation in the latent representation mostly reflects the local changes in the molecules, which is expected since changes in the scaffold typically lead to global changes such as chemical property changes.

\begin{figure}[t]
    \centering
    \includegraphics[width=0.44\textwidth]{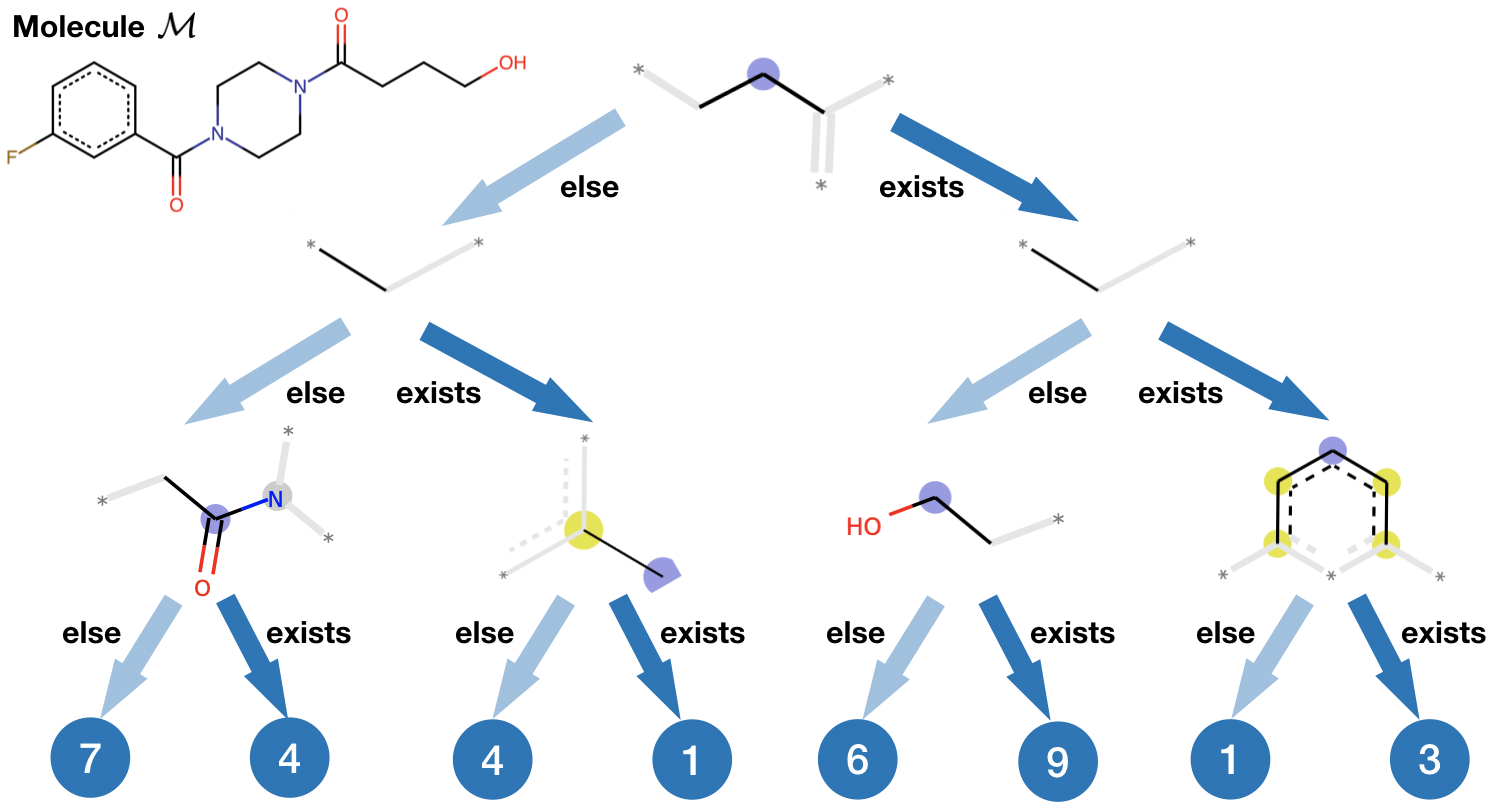}
    \vspace{-3mm}
    \caption{
    The local decision tree explains the latent representation for a molecule (upper left) by identifying locally discriminative chemical substructures. The leaf nodes are annotated with their sizes (number of molecules belonging to that cluster).}
    %The decision tree explaining the local neighborhood around a molecule $\mathcal{M}$ (top left). The neighborhood contains 35 molecules with similar structures, divided into 8 clusters (leaves) with various size. The molecule $\mathcal{M}$ falls into the red leaf with 9 neighboring compounds.
    \label{fig:molvae_tree}
    \vspace{-4mm}
\end{figure}

\begin{table}[t]
\setlength{\tabcolsep}{3.9pt}
\vspace{-3pt}
%\caption{ELBO and KL from the tree of the \textsc{Deep} and \textsc{Game} models.}
\caption{The performance in ELBO for the raw neural encoders and locally adapted decision trees. The deviation is defined in \xref{sec:molrepr}.}
\vspace{3pt}
\label{tab:molvae}
\centering
\begin{tabular}{ c | c c c }  
\toprule
\small Model & \small ELBO\tiny{neural encoder} & \small ELBO\tiny{decision tree} & \small deviation ($\mathcal{L}^{\mathcal{G}_\text{tree}}$) \\
\midrule
\small \textsc{Deep} & \small -21.6 & \small -25.4 & \small 4.64 \\
\small \textsc{Game} & \small \textbf{-21.5} & \small \textbf{-25.1} & \small \textbf{3.98} \\
\bottomrule
\end{tabular}
\vspace{-4mm}
\end{table}	
\vspace{-0.5mm}

\vspace{-2mm}
\section{Conclusion}
\vspace{-1mm}
\label{sec:conclusion}

We propose a novel game-theoretic approach to learning transparent models on structured data. The game articulates how the predictor model's fitting can be traded off against agreeing locally with a transparent witness. 
This work opens up many avenues for future work, from theoretical analysis of the games to a multi-player setting.

%In this work, the instantiated notions of transparency are based on a Markov assumption, which yield witnesses for short-term events.
%In this work, the instantiated notions of transparency are based on a single class of explanations. 
%It is interesting to extend it towards multi-facet transparency with a multi-player co-operative game among a predictor and multiple classes of witnesses. 
%This work opens up many avenues for future work, from theoretical analysis of the co-operative games to the estimation of transparent unfolded trajectories through GANs instead of single step dynamic models. 

\section*{Acknowledgement}
The work was funded in part by a grant from Siemens Corporation and in part by an MIT-IBM grant on deep rationalization.

% \bibliography{example_paper}
\bibliography{main.bbl}

\begin{thebibliography}{46}
\providecommand{\natexlab}[1]{#1}
\providecommand{\url}[1]{\texttt{#1}}
\expandafter\ifx\csname urlstyle\endcsname\relax
  \providecommand{\doi}[1]{doi: #1}\else
  \providecommand{\doi}{doi: \begingroup \urlstyle{rm}\Url}\fi

\bibitem[Abadi et~al.(2016)Abadi, Barham, Chen, Chen, Davis, Dean, Devin,
  Ghemawat, Irving, Isard, et~al.]{abadi2016tensorflow}
Abadi, M., Barham, P., Chen, J., Chen, Z., Davis, A., Dean, J., Devin, M.,
  Ghemawat, S., Irving, G., Isard, M., et~al.
\newblock Tensorflow: a system for large-scale machine learning.
\newblock In \emph{OSDI}, volume~16, pp.\  265--283, 2016.

\bibitem[Al-Shedivat et~al.(2017)Al-Shedivat, Dubey, and
  Xing]{al2017contextual}
Al-Shedivat, M., Dubey, A., and Xing, E.~P.
\newblock Contextual explanation networks.
\newblock \emph{arXiv preprint arXiv:1705.10301}, 2017.

\bibitem[Alvarez-Melis \& Jaakkola(2018{\natexlab{a}})Alvarez-Melis and
  Jaakkola]{alvarez2018towards}
Alvarez-Melis, D. and Jaakkola, T.
\newblock Towards robust interpretability with self-explaining neural networks.
\newblock In \emph{Advances in Neural Information Processing Systems}, pp.\
  7786--7795, 2018{\natexlab{a}}.

\bibitem[Alvarez-Melis \& Jaakkola(2017)Alvarez-Melis and
  Jaakkola]{alvarez2017causal}
Alvarez-Melis, D. and Jaakkola, T.~S.
\newblock A causal framework for explaining the predictions of black-box
  sequence-to-sequence models.
\newblock \emph{Proceedings of EMNLP}, 2017.

\bibitem[Alvarez-Melis \& Jaakkola(2018{\natexlab{b}})Alvarez-Melis and
  Jaakkola]{alvarez2018robustness}
Alvarez-Melis, D. and Jaakkola, T.~S.
\newblock On the robustness of interpretability methods.
\newblock \emph{arXiv preprint arXiv:1806.08049}, 2018{\natexlab{b}}.

\bibitem[Arjovsky et~al.(2017)Arjovsky, Chintala, and
  Bottou]{arjovsky2017wasserstein}
Arjovsky, M., Chintala, S., and Bottou, L.
\newblock Wasserstein gan.
\newblock \emph{arXiv preprint arXiv:1701.07875}, 2017.

\bibitem[Arras et~al.(2017)Arras, Horn, Montavon, M{\"{u}}ller, and
  Samek]{arras2017relevant}
Arras, L., Horn, F., Montavon, G., M{\"{u}}ller, K.-R., and Samek, W.
\newblock {" What is relevant in a text document?": An interpretable machine
  learning approach}.
\newblock \emph{PloS one}, 12\penalty0 (8):\penalty0 e0181142, 2017.

\bibitem[Bellemare et~al.(2017)Bellemare, Danihelka, Dabney, Mohamed,
  Lakshminarayanan, Hoyer, and Munos]{bellemare2017cramer}
Bellemare, M.~G., Danihelka, I., Dabney, W., Mohamed, S., Lakshminarayanan, B.,
  Hoyer, S., and Munos, R.
\newblock The cramer distance as a solution to biased wasserstein gradients.
\newblock \emph{arXiv preprint arXiv:1705.10743}, 2017.

\bibitem[Bemis \& Murcko(1996)Bemis and Murcko]{bemis1996properties}
Bemis, G.~W. and Murcko, M.~A.
\newblock The properties of known drugs. 1. molecular frameworks.
\newblock \emph{Journal of medicinal chemistry}, 39\penalty0 (15):\penalty0
  2887--2893, 1996.

\bibitem[Cisse et~al.(2017)Cisse, Bojanowski, Grave, Dauphin, and
  Usunier]{cisse2017parseval}
Cisse, M., Bojanowski, P., Grave, E., Dauphin, Y., and Usunier, N.
\newblock Parseval networks: Improving robustness to adversarial examples.
\newblock \emph{arXiv preprint arXiv:1704.08847}, 2017.

\bibitem[Citron \& Pasquale(2014)Citron and Pasquale]{citron2014scored}
Citron, D.~K. and Pasquale, F.
\newblock The scored society: due process for automated predictions.
\newblock \emph{Wash. L. Rev.}, 89:\penalty0 1, 2014.

\bibitem[Dai et~al.(2016)Dai, Dai, and Song]{dai2016discriminative}
Dai, H., Dai, B., and Song, L.
\newblock Discriminative embeddings of latent variable models for structured
  data.
\newblock In \emph{International Conference on Machine Learning}, pp.\
  2702--2711, 2016.

\bibitem[Ganin et~al.(2016)Ganin, Ustinova, Ajakan, Germain, Larochelle,
  Laviolette, Marchand, and Lempitsky]{ganin2016domain}
Ganin, Y., Ustinova, E., Ajakan, H., Germain, P., Larochelle, H., Laviolette,
  F., Marchand, M., and Lempitsky, V.
\newblock Domain-adversarial training of neural networks.
\newblock \emph{The Journal of Machine Learning Research}, 17\penalty0
  (1):\penalty0 2096--2030, 2016.

\bibitem[Ghorbani et~al.(2019)Ghorbani, Abid, and
  Zou]{ghorbani2017interpretation}
Ghorbani, A., Abid, A., and Zou, J.
\newblock Interpretation of neural networks is fragile.
\newblock \emph{AAAI}, 2019.

\bibitem[Goodfellow et~al.(2014)Goodfellow, Shlens, and Szegedy]{goodfellow14}
Goodfellow, I., Shlens, J., and Szegedy, C.
\newblock Explaining and harnessing adversarial examples.
\newblock 12 2014.

\bibitem[Griffen et~al.(2011)Griffen, Leach, Robb, and
  Warner]{griffen2011matched}
Griffen, E., Leach, A.~G., Robb, G.~R., and Warner, D.~J.
\newblock Matched molecular pairs as a medicinal chemistry tool:
  miniperspective.
\newblock \emph{Journal of medicinal chemistry}, 54\penalty0 (22):\penalty0
  7739--7750, 2011.

\bibitem[Gulrajani et~al.(2017)Gulrajani, Ahmed, Arjovsky, Dumoulin, and
  Courville]{gulrajani2017improved}
Gulrajani, I., Ahmed, F., Arjovsky, M., Dumoulin, V., and Courville, A.~C.
\newblock Improved training of wasserstein gans.
\newblock In \emph{Advances in Neural Information Processing Systems}, pp.\
  5767--5777, 2017.

\bibitem[He et~al.(2016)He, Zhang, Ren, and Sun]{he2016deep}
He, K., Zhang, X., Ren, S., and Sun, J.
\newblock Deep residual learning for image recognition.
\newblock In \emph{Proceedings of the IEEE conference on computer vision and
  pattern recognition}, pp.\  770--778, 2016.

\bibitem[Helfrich et~al.(2017)Helfrich, Willmott, and
  Ye]{helfrich2017orthogonal}
Helfrich, K., Willmott, D., and Ye, Q.
\newblock Orthogonal recurrent neural networks with scaled cayley transform.
\newblock \emph{arXiv preprint arXiv:1707.09520}, 2017.

\bibitem[Huang et~al.(2017)Huang, Liu, Van Der~Maaten, and
  Weinberger]{huang2017densely}
Huang, G., Liu, Z., Van Der~Maaten, L., and Weinberger, K.~Q.
\newblock Densely connected convolutional networks.
\newblock In \emph{CVPR}, volume~1, pp.\ ~3, 2017.

\bibitem[Jin et~al.(2018)Jin, Barzilay, and Jaakkola]{jin2018junction}
Jin, W., Barzilay, R., and Jaakkola, T.
\newblock Junction tree variational autoencoder for molecular graph generation.
\newblock \emph{arXiv preprint arXiv:1802.04364}, 2018.

\bibitem[Kingma \& Welling(2013)Kingma and Welling]{kingma2013auto}
Kingma, D.~P. and Welling, M.
\newblock Auto-encoding variational bayes.
\newblock \emph{arXiv preprint arXiv:1312.6114}, 2013.

\bibitem[Kova{\v{c}}evi{\'c} et~al.(2008)Kova{\v{c}}evi{\'c}, Chebira,
  et~al.]{kovavcevic2008introduction}
Kova{\v{c}}evi{\'c}, J., Chebira, A., et~al.
\newblock An introduction to frames.
\newblock \emph{Foundations and Trends{\textregistered} in Signal Processing},
  2\penalty0 (1):\penalty0 1--94, 2008.

\bibitem[Lakkaraju et~al.(2016)Lakkaraju, Bach, and
  Leskovec]{lakkaraju2016interpretable}
Lakkaraju, H., Bach, S.~H., and Leskovec, J.
\newblock Interpretable decision sets: A joint framework for description and
  prediction.
\newblock In \emph{Proceedings of the 22nd ACM SIGKDD international conference
  on knowledge discovery and data mining}, pp.\  1675--1684. ACM, 2016.

\bibitem[Lee et~al.()Lee, Alvarez-Melis, and Jaakkola]{lee2018game}
Lee, G.-H., Alvarez-Melis, D., and Jaakkola, T.~S.
\newblock Game-theoretic interpretability for temporal modeling.
\newblock \emph{The 5th Workshop on Fairness, Accountability, and Transparency
  in Machine Learning (FAT/ML 2018) at ICML 2018}.
\newblock URL \url{https://arxiv.org/pdf/1807.00130.pdf}.

\bibitem[Lee et~al.(2016)Lee, Qiu, Yu, Lin, and Rexnord Technical Services
  (2007).~IMS]{bearing_dataset}
Lee, J., Qiu, H., Yu, G., Lin, J., and Rexnord Technical Services (2007).~IMS,
  U. o.~C.
\newblock Bearing data set.
\newblock \emph{NASA Ames Prognostics Data Repository
  (http://ti.arc.nasa.gov/project/prognostic-data-repository), NASA Ames
  Research Center, Moffett Field, CA}, 7\penalty0 (8), 2016.

\bibitem[Lei et~al.(2016)Lei, Barzilay, and Jaakkola]{Lei2016Rationalizing}
Lei, T., Barzilay, R., and Jaakkola, T.
\newblock {Rationalizing Neural Predictions}.
\newblock In \emph{EMNLP 2016, Proceedings of the 2016 Conference on Empirical
  Methods in Natural Language Processing}, pp.\  107--117, 2016.
\newblock URL \url{http://arxiv.org/abs/1606.04155}.

\bibitem[Lei et~al.(2017)Lei, Jin, Barzilay, and Jaakkola]{lei2017deriving}
Lei, T., Jin, W., Barzilay, R., and Jaakkola, T.
\newblock Deriving neural architectures from sequence and graph kernels.
\newblock \emph{arXiv preprint arXiv:1705.09037}, 2017.

\bibitem[Mahendran \& Vedaldi(2015)Mahendran and
  Vedaldi]{mahendran2015understanding}
Mahendran, A. and Vedaldi, A.
\newblock Understanding deep image representations by inverting them.
\newblock \emph{Proceedings of the IEEE conference on computer vision and
  pattern recognition}, 2015.

\bibitem[Miller \& Hardt(2018)Miller and Hardt]{miller2018recurrent}
Miller, J. and Hardt, M.
\newblock When recurrent models don't need to be recurrent.
\newblock \emph{arXiv preprint arXiv:1805.10369}, 2018.

\bibitem[Mroueh et~al.(2018)Mroueh, Li, Sercu, Raj, and
  Cheng]{mroueh2017sobolev}
Mroueh, Y., Li, C.-L., Sercu, T., Raj, A., and Cheng, Y.
\newblock Sobolev gan.
\newblock \emph{International Conference on Learning Representations}, 2018.

\bibitem[M{\"u}ller(1997)]{muller1997integral}
M{\"u}ller, A.
\newblock Integral probability metrics and their generating classes of
  functions.
\newblock \emph{Advances in Applied Probability}, 29\penalty0 (2):\penalty0
  429--443, 1997.

\bibitem[Nemirovski(2004)]{nemirovski2004prox}
Nemirovski, A.
\newblock Prox-method with rate of convergence o (1/t) for variational
  inequalities with lipschitz continuous monotone operators and smooth
  convex-concave saddle point problems.
\newblock \emph{SIAM Journal on Optimization}, 15\penalty0 (1):\penalty0
  229--251, 2004.

\bibitem[Pasquale(2015)]{pasquale2015black}
Pasquale, F.
\newblock \emph{The black box society: The secret algorithms that control money
  and information}.
\newblock Harvard University Press, 2015.

\bibitem[Pedregosa et~al.(2011)Pedregosa, Varoquaux, Gramfort, Michel, Thirion,
  Grisel, Blondel, Prettenhofer, Weiss, Dubourg, et~al.]{pedregosa2011scikit}
Pedregosa, F., Varoquaux, G., Gramfort, A., Michel, V., Thirion, B., Grisel,
  O., Blondel, M., Prettenhofer, P., Weiss, R., Dubourg, V., et~al.
\newblock Scikit-learn: Machine learning in python.
\newblock \emph{Journal of machine learning research}, 12\penalty0
  (Oct):\penalty0 2825--2830, 2011.

\bibitem[Quinlan(2014)]{quinlan2014c4}
Quinlan, J.~R.
\newblock \emph{C4. 5: programs for machine learning}.
\newblock Elsevier, 2014.

\bibitem[Ribeiro et~al.(2016)Ribeiro, Singh, and Guestrin]{Ribeiro2016Why}
Ribeiro, M.~T., Singh, S., and Guestrin, C.
\newblock {"Why Should I Trust You?": Explaining the Predictions of Any
  Classifier}.
\newblock In \emph{Proceedings of the 22Nd ACM SIGKDD International Conference
  on Knowledge Discovery and Data Mining}, pp.\  1135--1144, New York, NY, USA,
  2016. ACM.
\newblock ISBN 978-1-4503-4232-2.
\newblock \doi{10.1145/2939672.2939778}.
\newblock URL \url{http://arxiv.org/abs/1602.04938
  http://doi.acm.org/10.1145/2939672.2939778}.

\bibitem[Rogers \& Hahn(2010)Rogers and Hahn]{rogers2010extended}
Rogers, D. and Hahn, M.
\newblock Extended-connectivity fingerprints.
\newblock \emph{Journal of chemical information and modeling}, 50\penalty0
  (5):\penalty0 742--754, 2010.

\bibitem[Selvaraju et~al.(2016)Selvaraju, Cogswell, Das, Vedantam, Parikh, and
  Batra]{selvaraju2016grad}
Selvaraju, R.~R., Cogswell, M., Das, A., Vedantam, R., Parikh, D., and Batra,
  D.
\newblock Grad-cam: Visual explanations from deep networks via gradient-based
  localization.
\newblock \emph{https://arxiv. org/abs/1610.02391 v3}, 7\penalty0 (8), 2016.

\bibitem[Shamir(2015)]{shamir2015sample}
Shamir, O.
\newblock The sample complexity of learning linear predictors with the squared
  loss.
\newblock \emph{The Journal of Machine Learning Research}, 16\penalty0
  (1):\penalty0 3475--3486, 2015.

\bibitem[Silver et~al.(2016)Silver, Huang, Maddison, Guez, Sifre, Van
  Den~Driessche, Schrittwieser, Antonoglou, Panneershelvam, Lanctot,
  et~al.]{silver2016mastering}
Silver, D., Huang, A., Maddison, C.~J., Guez, A., Sifre, L., Van Den~Driessche,
  G., Schrittwieser, J., Antonoglou, I., Panneershelvam, V., Lanctot, M.,
  et~al.
\newblock Mastering the game of go with deep neural networks and tree search.
\newblock \emph{nature}, 529\penalty0 (7587):\penalty0 484, 2016.

\bibitem[Sterling \& Irwin(2015)Sterling and Irwin]{sterling2015zinc}
Sterling, T. and Irwin, J.~J.
\newblock Zinc 15--ligand discovery for everyone.
\newblock \emph{Journal of chemical information and modeling}, 55\penalty0
  (11):\penalty0 2324--2337, 2015.

\bibitem[Vaswani et~al.(2017)Vaswani, Shazeer, Parmar, Uszkoreit, Jones, Gomez,
  Kaiser, and Polosukhin]{vaswani2017attention}
Vaswani, A., Shazeer, N., Parmar, N., Uszkoreit, J., Jones, L., Gomez, A.~N.,
  Kaiser, {\L}., and Polosukhin, I.
\newblock Attention is all you need.
\newblock In \emph{Advances in Neural Information Processing Systems}, pp.\
  5998--6008, 2017.

\bibitem[Wu et~al.(2018{\natexlab{a}})Wu, Hughes, Parbhoo, Zazzi, Roth, and
  Doshi{-}Velez]{wu2018tree}
Wu, M., Hughes, M.~C., Parbhoo, S., Zazzi, M., Roth, V., and Doshi{-}Velez, F.
\newblock Beyond sparsity: Tree regularization of deep models for
  interpretability.
\newblock In \emph{Proceedings of the Thirty-Second {AAAI} Conference on
  Artificial Intelligence, New Orleans, Louisiana, USA, February 2-7, 2018},
  2018{\natexlab{a}}.
\newblock URL
  \url{https://www.aaai.org/ocs/index.php/AAAI/AAAI18/paper/view/16285}.

\bibitem[Wu et~al.(2018{\natexlab{b}})Wu, Ramsundar, Feinberg, Gomes, Geniesse,
  Pappu, Leswing, and Pande]{wu2018moleculenet}
Wu, Z., Ramsundar, B., Feinberg, E.~N., Gomes, J., Geniesse, C., Pappu, A.~S.,
  Leswing, K., and Pande, V.
\newblock Moleculenet: a benchmark for molecular machine learning.
\newblock \emph{Chemical science}, 9\penalty0 (2):\penalty0 513--530,
  2018{\natexlab{b}}.

\bibitem[Zhao et~al.(2017)Zhao, Yue, Katabi, Jaakkola, and
  Bianchi]{zhao2017learning}
Zhao, M., Yue, S., Katabi, D., Jaakkola, T.~S., and Bianchi, M.~T.
\newblock Learning sleep stages from radio signals: a conditional adversarial
  architecture.
\newblock In \emph{International Conference on Machine Learning}, pp.\
  4100--4109, 2017.

\end{thebibliography}
\bibliographystyle{icml2019}

\clearpage
%\newpage
\appendix
\section{Proofs}\label{sec:proof}

Our main results in this section make the following assumptions.

\begin{itemize}[label={}, leftmargin=-0.5mm]
\vspace{-3.5mm}
\setlength\itemsep{-0.02em}
  \item \header{(A1)} the predictor $f$ is unconstrained.
  \vspace{-1mm}
  \item \header{(A2)} both the loss and deviation are squared errors.
  \vspace{-1mm}
  \item \header{(A3)} $|\mathcal{B}(x_i)| = m, \forall x_i \in \mathcal{D}_x$.
  \vspace{-1mm}
  \item \header{(A4)} $x_j \in \mathcal{B}(x_i) \implies x_i \in \mathcal{B}(x_j), \forall x_i, x_j \in \mathcal{D}_x$.
  \vspace{-1mm}
  \item \header{(A5)} $\cup_{x_i \in \mathcal{D}_x} \mathcal{B}(x_i) = \mathcal{D}_x$.
\vspace{-2.5mm}
\end{itemize}
We note that \textbf{(A3)} and \textbf{(A4)} are not technically necessary but simplify the presentation. We denote the predictor in the uniform criterion (Eq.~\eqref{eq:uniform:game}), the symmetric game (Eq.~\eqref{eq:coop_game}), and the asymmetric (Eq.~\eqref{eq:asym_coop_game}) game as $f_U$, $f_S$, and $f_A$, respectively. 
We use $X_i\in \mathbb{R}^{m\times d}$ to denote the neighborhood $\mathcal{B}(x_i) = \{x'_1,\dots,x'_m\}$ ($X_i = [x'_1,\dots,x'_m]^\top$), and $f(X_i)\in \mathbb{R}^m$ to denote the vector $[f(x'_1),\dots,f(x'_m)]^\top$. $X_j^\dagger$ denotes the pseudo-inverse of $X_j$. Then we have
\begin{theorem5*}
If \emph{\textbf{(A1-5)}} hold and the witness is in the linear family, %Eq. (\ref{eq:opt_s_const}) is the unique equilibrium for $f_S$
the optimal $f_S$ satisfies 
\vspace{-2mm}
\begin{equation}
f^*_S(x_i) = \frac{1}{1+\lambda} \bigg[ y_i + \frac{\lambda}{m} \bigl(\sum_{x_j \in \mathcal{B}(x_i)} X^\dagger_j f^*_S(X_j) \bigr)^\top x_i \bigg], \nonumber
%\label{eq:opt_s_const}
\vspace{-2mm}
\end{equation}
and the optimal $f_A$, at every equilibrium, is the fixed point
\begin{equation}
\vspace{-2mm}
f^*_A(x_i) = \frac{1}{1+\lambda} \bigg[ y_i + \lambda  (X^\dagger_if^*_A(X_i) )^\top  x_i \bigg], \forall x_i \in \mathcal{D}_x. \nonumber %\label{eq:opt_a_const}%\nonumber
% \vspace{-2mm}
\end{equation}
\end{theorem5*}

%%%%%%%%%%%%%% proof begins

\begin{proof}
% For the symmetric game, our strategy is to first treat each $g$ as a fixed function $\hat{g}_{x_i}$, and then replace it with its best response strategy. 
We first re-write the symmetric criterion explicitly as a game:
\begin{equation}
\min_f \sum_i (f(x_i) - y_i)^2 + \frac{\lambda}{m} \sum_{x_j \in \mathcal{B}(x_i)} (f(x_j) - \hat{g}_{x_i}(x_j))^2, \nonumber
%\label{eq:symmetric}
\end{equation}
where $\hat{g}_{x_i}$ is the best response strategy from the local witness. 

Since $f$ is unconstrained and the objective in convex in it, we can treat each $f(x_i)$ as a distinct variable, and use the derivative to find its optimum:
\begin{align}
f^*_S(x_i) & = \frac{1}{1+\lambda}\bigg[ y_i + \frac{\lambda}{m} \sum_{x_j \in \mathcal{B}^{-1}(x_i)} \hat{g}_{x_j}(x_i) \bigg] \nonumber\\
& = \frac{1}{1+\lambda}\bigg[ y_i + \frac{\lambda}{m} \sum_{x_j \in \mathcal{B}(x_i)} \hat{g}_{x_j}(x_i) \bigg],
\label{eq:optimal_sym}
\end{align}
where $\mathcal{B}^{-1}(x_i) = \{x_j \in \mathcal{D}_x: x_i \in \mathcal{B}(x_j)\}$.
Note that we only have to collect witnesses $\hat{g}_{x_j}$ that are relevant to $f(x_i)$ for the first equality, and the second equality is due to \textbf{(A4)}. 
On the other hand, the objective for $f$ in the asymmetric game is:
\begin{equation}
\min_f \sum_i (f(x_i) - y_i)^2 + \lambda (f(x_i) - \hat{g}_{x_i}(x_i))^2, \nonumber
% \label{eq:asymmetric}
\end{equation}
The corresponding optimum is:
\begin{equation}
f^*_A(x_i) = \frac{1}{1+\lambda}\bigg[ y_i + \lambda \hat{g}_{x_i}(x_i) \bigg]
\label{eq:optimal_asym}
\end{equation}

For both games, the objective for $g_{x_i}$ can be described as:
\begin{align}
& \min_{g_{x_i}} \frac{\lambda}{m} \sum_{x_j \in \mathcal{B}(x_i)} (f(x_j) - g_{x_i}(x_j))^2 \nonumber\\
& = \min_{\theta_i} \frac{\lambda}{m} \|f(X_i) - X_i \theta_i \|_2^2, %\nonumber
\label{eq:refer_game_linear}
\end{align}
Then Eq. (\ref{eq:optimal_linear}) is an optimal witness $g^*_{x_i}$ at $x_i$.
\begin{equation}
g^*_{x_i}(x_j) = \theta^\top_i x_j  = (X^\dagger_i f(X_i))^\top x_j, \forall x_j \in \mathcal{X},
\label{eq:optimal_linear}
\end{equation}
and we note that every optimal witness $g^*_{x_i}$ has the same values on $\mathcal{B}(x_i)$

Since the optimal $g^*_{x_i}$ is functionally dependent to $f$. we put Eq. (\ref{eq:optimal_linear}) back to Eq. (\ref{eq:optimal_sym}) to obtain the optimal condition for $f^*_S$ (at equilibrium) as 
\begin{align}
&f^*_S(x_i) = \frac{1}{1+\lambda} \bigg[ y_i + \frac{\lambda}{m} (\sum_{x_j \in \mathcal{B}(x_i)} X^\dagger_j f^*_S(X_j) )^\top x_i \bigg].\nonumber
\label{eq:expand_dplinear_sym}
\end{align}
Again, putting Eq. (\ref{eq:optimal_linear}) back to Eq. (\ref{eq:optimal_asym}), we obtain the optimal condition for $f^*_A$ at equilibrium as 
\begin{align}
&f^*_A(x_i) = \frac{1}{1+\lambda} \bigg[ y_i + \lambda  (X^\dagger_if^*_A(X_i) )^\top  x_i \bigg].\nonumber
% \label{eq:expand_dplinear_asym}
\end{align}
\end{proof}

%Theorem~\ref{lemma:simple_case} states that both games induce recursive convolutional averaging of neighboring points with the same decay rate $\frac{\lambda}{1+\lambda}$, while the convolutional kernel evolves twice faster in the symmetric game than in the asymmetric game.

Note that the equilibrium for the linear class is not unique when the solution of Eq.~(\ref{eq:refer_game_linear}) is not unique: there may be infinitely many optimal solution to the witness in a neighborhood due to degeneracy. In this case, Theorem \ref{lemma:linear_case} adopts the minimum norm solution as used in the pseudo-inverse in Eq.~(\ref{eq:optimal_linear}). 
In this case, one may use Ridge regression instead to establish a strongly convex objective for the witness to ensure a unique solution, where the objective for the witness is rewritten as
\begin{equation}
    \min_{\theta_i} \frac{\lambda}{m} \|f(X_i) - X_i\theta_i\|^2_2 + \alpha \|\theta_i\|^2_2,
\end{equation}
with a positive $\alpha$. 

%%%%%%%%%%%%%% proof ends %%%%%%%%%%%%%%%%%%%
%As a result, we can follow the same argument as the proof of Theorem~\ref{lemma:simple_case} to prove the uniqueness of equilibrium.

\begin{theorem6*}
If \emph{\textbf{(A1-5)}} hold and the witness is in the linear family, the optimal $f_U$ satisfies
\begin{align}
    f_U^*(x_i) = \left \{
  \begin{aligned}
    & \alpha(x_i, f_U^*), && \text{if}\ \alpha(x_i, f_U^*) > y_i, \\
    & \beta(x_i, f_U^*), && \text{if}\ \beta(x_i, f_U^*) < y_i, \\
    & y_i, && \text{otherwise,}
  \end{aligned} \right. \nonumber
\end{align}
for $x_i \in \mathcal{D}_x$, where
\begin{align*}
    & \alpha(x_i, f^*_U) = \max_{x_j \in \mathcal{B}(x_i)} \bigg[ (X_j^\dagger f^*_U(X_j))^\top x_i \\
    & - \sqrt{\delta m - \sum_{x_k \in \mathcal{B}(x_j)\backslash \{x_i\} }  (f^*_U(x_k) -  (X_j^\dagger f^*_U(X_j))^\top x_k )^2} \bigg];\\
    & \beta(x_i, f^*_U) = \min_{x_j \in \mathcal{B}(x_i)} \bigg[ (X_j^\dagger f^*_U(X_j))^\top x_i \\
    & + \sqrt{\delta m - \sum_{x_k \in \mathcal{B}(x_j)\backslash \{x_i\} }  (f^*_U(x_k) -  (X_j^\dagger f^*_U(X_j))^\top x_k )^2} \bigg].
\end{align*}
\label{theorem:appendix:uniform}
\end{theorem6*}

\begin{proof}
The objective for the uniform criterion is:
\begin{align}
    & \min_f \sum_{i=1}^N (f(x_i) - y_i)^2 \label{eq:unif:orig:game}\\
    & s.t.\;\; \min_{g \in \mathcal{G}} \frac{1}{m} \sum_{x_j \in \mathcal{B}(x_i)} (f(x_j) - g(x_j))^2 \leq \delta, \forall x_i \in \mathcal{D}_x. \nonumber
\end{align}
\iffalse
We can rewrite it as:
\begin{align}
    & \min_{f, g_{x_i} \in \mathcal{G}} \sum_{i=1}^N (f(x_i) - y_i)^2 \label{eq:app:rewrite}\\
    & s.t.\;\; \frac{1}{m} \sum_{x_j \in \mathcal{B}(x_i)} (f(x_j) - g_{x_i}(x_j))^2 \leq \delta, \forall x_i \in \mathcal{D}_x. \nonumber
\end{align}
As the optimal $f$ at every equilibrium of Eq.~(\ref{eq:app:rewrite}) is optimal for Eq.~(\ref{eq:unif:orig:game}), and vice versa, we can study the equilibrium of Eq.~(\ref{eq:app:rewrite}) for the optimality condition of the uniform criterion.
\fi
Our strategy is to temporarily treat each $g$ as a fixed function, and then replace it with its best response strategy. 

Since $f$ is unconstrained (in capacity), we can treat each $f(x_i)$ as a distinct variable for optimization.
For each $f(x_i)$, we first filter its relevant criteria:
\begin{align}
    \min_{f(x_i)} & (f(x_i) - y_i)^2 \nonumber\\
    s.t.\;\; & (f(x_i) - g_{x_j}(x_i))^2, \leq \delta m \nonumber \\
    & - \sum_{x_k \in \mathcal{B}(x_j) \backslash \{x_i\} } (f(x_k) - g_{x_j}(x_k))^2, \forall x_j \in \mathcal{B}(x_i). \nonumber
\end{align}
For any feasible $f$, we can further rewrite the constraint of $f(x_i)$ with respect to each $x_j$ as:
\begin{align*}
& g_{x_j}(x_i) - \sqrt{\delta m - \sum_{x_k \in \mathcal{B}(x_j)\backslash \{x_i\} } (f(x_k) - g_{x_j} (x_k))^2} \\
& \leq f(x_i) \\
& \leq g_{x_j}(x_i) + \sqrt{\delta m - \sum_{x_k \in \mathcal{B}(x_j)\backslash \{x_i\} } (f(x_k) - g_{x_j} (x_k))^2}.
\end{align*}
Collectively, we can fold all the upper bounds of $f(x_i)$ as
\begin{align*}
    & f(x_i) \leq \min_{x_j \in \mathcal{B}(x_i)} \bigg[ g_{x_j}(x_i) \\
    & + \sqrt{\delta m - \sum_{x_k \in \mathcal{B}(x_j)\backslash \{x_i\} } (f(x_k) - g_{x_j} (x_k))^2} \bigg].
\end{align*}
All the lower bounds can be folded similarly. 

Finally, since the objective for $f(x_i)$ is simply a squared error with an interval constraint, evidently if $y_i$ satisfies the lower bounds and upper bounds, then $f_U^*(x_i) = y_i$. If 
\begin{align*}
    & y_i > \min_{x_j \in \mathcal{B}(x_i)} \bigg[ g_{x_j}(x_i) \\
    & + \sqrt{\delta m - \sum_{x_k \in \mathcal{B}(x_j)\backslash \{x_i\} } (f(x_k) - g_{x_j} (x_k))^2} \bigg],
\end{align*}
then we have 
\begin{align*}
    & f_U^*(x_i) = \min_{x_j \in \mathcal{B}(x_i)} \bigg[ g_{x_j}(x_i) \\
    & + \sqrt{\delta m - \sum_{x_k \in \mathcal{B}(x_j)\backslash \{x_i\} } (f(x_k) - g_{x_j} (x_k))^2} \bigg].
\end{align*}
Otherwise, we have
\begin{align*}
    & f_U^*(x_i) = \max_{x_j \in \mathcal{B}(x_i)} \bigg[ g_{x_j}(x_i) \\
    & - \sqrt{\delta m - \sum_{x_k \in \mathcal{B}(x_j)\backslash \{x_i\} } (f(x_k) - g_{x_j} (x_k))^2} \bigg].
\end{align*}

For each $g_{x_i}$ is in the linear class, Eq.~(\ref{app:unif:linear}) is an optimal solution. 
\begin{equation}
    g^*_{x_j}(x_i) = (X_j^\dagger f(X_j))^\top x_i, \forall x_i \in \mathcal{X}, \label{app:unif:linear}
\end{equation}
and we note that every optimal witness $g^*_{x_j}$ has the same values on $\mathcal{B}(x_j)$.

Since the optimal $g^*_{x_i}$ is functionally dependent to $f$, to obtain the optimal $f^*_U$, we combine our previous result with $g^*_{x_i}$ such that the optimality conditions for $f$ and $g_{x_i}$ are both satisfied. Finally, we have
\begin{align}
    f_U^*(x_i) = \left \{
  \begin{aligned}
    & \alpha(x_i, f_U^*), && \text{if}\ \alpha(x_i, f_U^*) > y_i, \\
    & \beta(x_i, f_U^*), && \text{if}\ \beta(x_i, f_U^*) < y_i, \\
    & y_i, && \text{otherwise,}
  \end{aligned} \right. \nonumber
\end{align}
for $x_i \in \mathcal{D}_x$, where
\begin{align*}
    & \alpha(x_i, f^*_U) = \max_{x_j \in \mathcal{B}(x_i)} \bigg[ (X_j^\dagger f^*_U(X_j))^\top x_i \\
    & - \sqrt{\delta m - \sum_{x_k \in \mathcal{B}(x_j)\backslash \{x_i\} }  (f^*_U(x_k) -  (X_j^\dagger f^*_U(X_j))^\top x_k )^2} \bigg];\\
    & \beta(x_i, f^*_U) = \min_{x_j \in \mathcal{B}(x_i)} \bigg[ (X_j^\dagger f^*_U(X_j))^\top x_i \\
    & + \sqrt{\delta m - \sum_{x_k \in \mathcal{B}(x_j)\backslash \{x_i\} }  (f^*_U(x_k) -  (X_j^\dagger f^*_U(X_j))^\top x_k )^2} \bigg].\\
\end{align*}
\end{proof}

\begin{lemma7*}
If $d(\cdot, \cdot)$ is squared error, $\mathcal{L}(\cdot, \cdot)$ is differentiable, $f$ is sub-differentiable, and \emph{$\textbf{A(4-5)}$} hold, then 
\vspace{-1.5mm}
\begin{equation}
\sum_{(x_i, y_i) \in \mathcal{D}}\! \mathcal{L}(f(x_i), y_i) + \frac{\lambda}{\bar{N}_i}\bigg[ \bar{N}_i f(x_i) - \!\!\!\! \sum_{x_t\in \mathcal{B}(x_i)} \! \frac{\hat{g}_{x_t}(x_i)}{|\mathcal{B}(x_t)|} \bigg]^2\!\!\!, \! \label{eq:adjusted_game}
\end{equation}
where $\bar{N}_i := \sum_{x_t \in \mathcal{B}(x_i)}\! \frac{1}{|\mathcal{B}(x_t)|}$, induces the same equilibrium as the symmetric game. 
\end{lemma7*}

\begin{proof}
Since the criteria for the witness $g_{x_i}$ are the same in the symmetric game and the proposed asymmetric criterion here, we only have to check for the optimality condition for the predictor $f$. If we use $\nabla_\theta f(x)$ to denote the subgradient of $f$ at $x$ with respect to the underlying parameter $\theta$, the optimality condition for Eq.~(\ref{eq:adjusted_game}) is 
\begin{align}
    0 \in & \sum_{(x_i, y_i) \in \mathcal{D}} \bigg[   \frac{\partial }{\partial f(x_i)} \mathcal{L}(f(x_i), y_i) \nonumber\\
    & + 2 \lambda (\sum_{x_t \in \mathcal{B} (x_i)} \frac{f(x_i)}{|\mathcal{B}(x_t)|} - \sum_{x_t \in \mathcal{B}(x_i)} \frac{\hat{g}_{x_t} (x_i)}{|\mathcal{B}(x_t)|} )   \bigg] \nabla_\theta f(x_i) \nonumber\\
    = &  \sum_{(x_i, y_i) \in \mathcal{D}} \bigg[ \frac{\partial }{\partial f(x_i)} \mathcal{L}(f(x_i), y_i) \nabla_\theta f(x_i) \nonumber\\
    & + \sum_{x_t \in \mathcal{B} (x_i)} \frac{2 \lambda}{|\mathcal{B}(x_t)|}(f(x_i) - \hat{g}_{x_t} (x_i)) \nabla_\theta f(x_i)   \bigg] \nonumber
\end{align}
For the symmetric game, the optimality condition is
\begin{align}
    0 \in & \sum_{(x_i, y_i) \in \mathcal{D}} \bigg[  \frac{\partial}{\partial f(x_i)} \mathcal{L}(f(x_i), y_i) \nabla_\theta f(x_i) \nonumber \\
    & + \sum_{x_t \in \mathcal{B}(x_i)} \frac{2\lambda}{|\mathcal{B}(x_i)|} (f(x_t) - \hat{g}_{x_i}(x_t)) \nabla_\theta f(x_t) \bigg] \nonumber
\end{align}
It is evident that the two conditions coincide if Eq.~(\ref{eq1}) is equal to Eq.~(\ref{eq4}).
\begin{align}
    & \sum_{(x_i, y_i) \in \mathcal{D}} \sum_{x_t \in \mathcal{B}(x_i)} \frac{1}{|\mathcal{B}(x_i)|} (f(x_t) - \hat{g}_{x_i}(x_t))\label{eq1} \nabla_\theta f(x_t)\\
    & = \sum_{x_t \in \cup_{x_i \in \mathcal{D}_x}\mathcal{B}(x_i)} \sum_{x_i \in \mathcal{B}^{-1}(x_t)} \nonumber \\
    & \;\;\;\;\;\;\;\;\;\;\;\;\;\;\;\;\;\;\;\;\;\;\;\;\;\;\frac{1}{|\mathcal{B}(x_i)|} (f(x_t) - \hat{g}_{x_i}(x_t)) \nabla_\theta f(x_t)\nonumber\\
    & = \sum_{x_t \in \mathcal{D}_x} \sum_{x_i \in \mathcal{B}(x_t)} \frac{1}{|\mathcal{B}(x_i)|} (f(x_t) - \hat{g}_{x_i}(x_t)) \nabla_\theta f(x_t)\nonumber\\
    & = \sum_{(x_i, y_i) \in \mathcal{D}} \sum_{x_t \in \mathcal{B}(x_i)} \frac{1}{|\mathcal{B}(x_t)|} (f(x_i) - \hat{g}_{x_t} (x_i)) \nabla_\theta f(x_i) \label{eq4},
\end{align}
where the first equality is simply re-ordering of the two summations, and the second equality is due to $x_t \in\mathcal{B}(x_i) \iff x_i \in \mathcal{B}(x_t)$ and $\cup_{x_i \in \mathcal{D}_x}\mathcal{B}(x_i) = \mathcal{D}_x$.
\end{proof}

\begin{figure}[t]
\centering
  \centering 
  \includegraphics[width=0.9\linewidth]{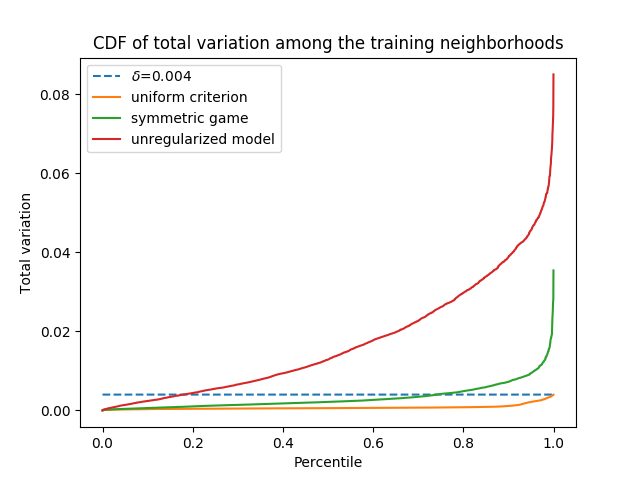}
  \caption{The cumulative distribution function of the total variation loss between the predictor $f$ and the local witness $g$ in each training neighborhood.}
  \label{fig:train:tv}
\end{figure}

\begin{figure*}[t]
\centering
  \includegraphics[width=1.\linewidth]{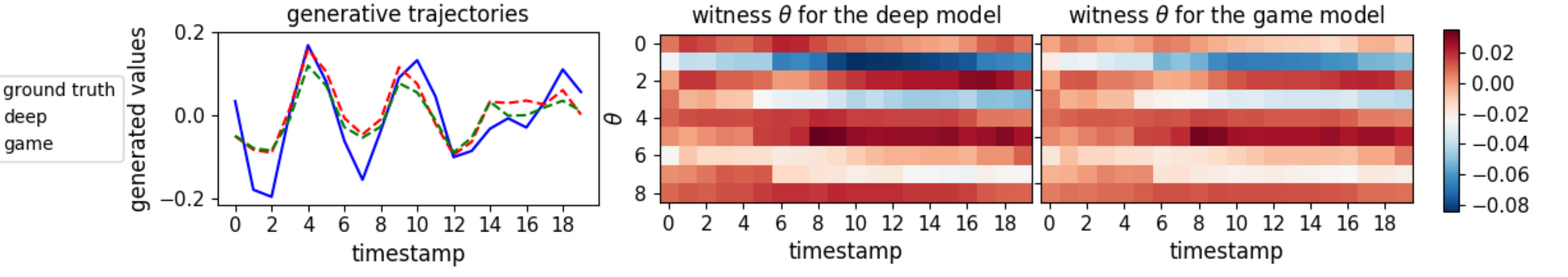}
  \caption{Visualization of the witnesses with the their parameters (middle and right plots) for \emph{teacher-forced} predictions on the first channel (left plot) along each timestamp ($x$-axis) on the bearing dataset. The $y$-axis of the parameters from $0$ to $8$ denotes the bias $(\theta_0)_1$ and weights $(\theta_1)_{1, 1:4}, (\theta_2)_{1, 1:4}$.}
  \label{fig:visual_tf}
%\vspace{-3mm}
\end{figure*}

% \section{Discussion for the Games}\label{app:discuss:game}

% \begin{figure}	
%   \centering
%   \includegraphics[width=1.\linewidth]{img/game_graph.pdf}

% 	\caption{
% 	Illustration of the difference between our game versus GANs~\citep{goodfellow2014generative}. 
%     }\label{fig:game_graph}
%     \vspace{-1mm}
% \end{figure}

% The co-operative nature (min-min) of the symmetric game makes the formulation ideal for optimization. First, progress can be immediately determined from the value of the objective. In contrast, measuring the progress of optimization for saddle point equilibria in minimax games such as in GANs is challenging~\citep{goodfellow2014generative}. Methodologically, we are essentially aligning a flexible predictor (P) with both labeled data (Y) and a functionally transparent witness (W) via a loss function and a fixed deviation function, respectively, while GANs align a flexible generator (G) with unlabeled data (X) via an adversarially learned discriminator. The difference is illustrated in Figure~\ref{fig:game_graph}. Clearly, a possible generalization would be to replace the deviation function in our game with a discriminator, which can be useful for distributional outputs.
% %} 

\section{Supplementary Materials for Molecule Property Prediction}\label{appendix:molecule}

%For each data point, we construct the neighborhood by searching the xxx unlabeled molecule database by xxx rules, which ensures the constructed neighborhoods are chemically valid. 

%Our implementation is based on a public PyTorch~\cite{paszke2017automatic} implementation for GCNs.
\header{Implementation.}
To conduct training, we use GCNs as the predictor with 6 layers of graph convolution with $1800$ hidden dimension. We use a $80\% / 10\% / 10\%$ split for training / validation / testing.
% The symmetric game is adopted.
%We use the symmetric game since the neighborhoods are constructed from the unlabeled data: the deviation is the only guidance for the unlabeled data.
% The batch size is set to $1$. The learning rate is set to $10^{-4}$ with the AMSGrad optimizer~\cite{reddi2018convergence}. We train the model for $15$ epochs. We only tune $\lambda \in \{10^{-2}, 10^{-1}, \dots\}$ until significant performance drop occurs. The tuned $\lambda$ is $10$.

% Training for the \textsc{Deep} model takes about half a day, and training for the \textsc{Game} model takes about $1$ day on one Titan X GPU. 

\header{Visualization.}
To investigate the behavior of the models, we plot their total variation loss from the local witness among the training neighborhoods in Figure~\ref{fig:train:tv}. The uniform criterion imposes a strict functional constraint, while the symmetric game allows a more flexible model, exhibiting a tiny fraction of high deviation among the training neighborhoods.

\section{Supplementary Materials for Physical Component Modeling}~\label{app:time_series}

\header{Implementation.}
We randomly sample $85\%$, $5\%$, and $10\%$ of the data for training, validation, and testing. 
% We set the learning rate as $10^{-5}$ with the Adam optimizer~\cite{kingma2014adam}. The batch size is set to $128$. 
All the hidden dimensions are set to $128$. We use the \texttt{MultivariateNormalTriL} function in Tensorflow~\cite{abadi2016tensorflow} to parametrize the multivariate Gaussian distribution. Specifically, we let the network output a $N + \frac{(N+1)(N)}{2}$ dimensional vector. The first $N$ dimensions are treated as the mean. The second part is transformed to a lower triangular matrix, where the diagonal is further processed with a softplus nonlinearity. Such representation satisfies the Cholesky decomposition for covariance matrix. 

For fitting the linear witness, we use Ridge regression in \texttt{scikit-learn}~\cite{pedregosa2011scikit} with the default hyperparameter. The usage of Ridge regression instead of vanilla linear regression is justified by our analysis of the equilibrium for linear witnesses.

\header{Visualization.}
The visualization for the teacher-forced generative trajectory is in Figure~\ref{fig:visual_tf}.

\header{Neighborhood size analysis}

\begin{figure}[tbpb]
\centering
  \centering 
  \includegraphics[width=0.9\linewidth]{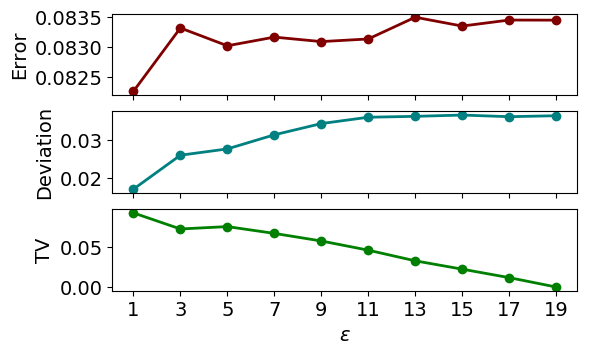}
  \caption{Parameter analysis of $\epsilon$ on the \textsc{Game} model with $\lambda=1$.}
  \label{fig:neighborhood_size}
\end{figure}

Here we investigate the effect of neighborhood radius $\epsilon$. The results are shown in Figure~\ref{fig:neighborhood_size}. The impact of the neighborhood size is quite monotonic to deviation and TV, but in a reverse way. As $\epsilon$ increases, the weight of the witness on fitting the current point $x_i$ among the neighborhood $\mathcal{B}(x_i)$ decreases, so the deviation of the witness $\hat{g}_{x_i}(x_i)$ from $f(x_i)$ increases. In contrast, as more points are overlapped between the neighborhoods of consecutive points, the resulting witnesses are more similar and thus yield smaller TV. In terms of prediction error, as the neighborhood radius $\epsilon$ determines the region to impose coherency, a larger region leads to greater restriction on the predictive model. All the arguments are well supported by the empirical results. We suggest users to trade off {faithfulness (deviation) and smooth transition of functional properties (TV)} based on the application at hand. We note that, however, smooth transition of functional properties is not equivalent to smoothness of $f$.

Finally, we remark that our sample complexity analysis for the linear class suggests that the neighborhood size is guaranteed to be effective when $2 \epsilon + 1 > d = 2c + 1 = 9$. However, since the result is an sufficient condition, the regularization may still happens when $\epsilon < 5$ (e.g., if the matrix rank of a neighborhood $X_i = [x_{i-\epsilon},\dots,x_{i+\epsilon}]^\top$ is less than $\min\{d, m\} = \min\{2c + 1, 2\epsilon + 1\}$). 

%%%%%%%%%%%%%%%%%%%%%%%%%%%%%%%%%%%%%%%%%%%%%%%%%%%%%%%%%%%%%%%%%%%%%%%%%%%%%%%
%%%%%%%%%%%%%%%%%%%%%%%%%%%%%%%%%%%%%%%%%%%%%%%%%%%%%%%%%%%%%%%%%%%%%%%%%%%%%%%

\end{document}